\documentclass[11pt, a4paper]{article}
\usepackage[utf8]{inputenc}
\usepackage{authblk}
\usepackage{enumitem}
\usepackage[babel=true]{csquotes}
\usepackage{caption}
\usepackage{subcaption}
\usepackage{color}
\usepackage{dsfont}
\usepackage{booktabs}
\usepackage{capt-of}
\usepackage{multicol}

\usepackage{appendix}
\usepackage{float}
\usepackage{url}
\usepackage[colorlinks=true,bookmarks=false,linkcolor=blue,citecolor=blue,urlcolor=blue]{hyperref}
\usepackage{array,tabu,multirow,makecell}
\usepackage{threeparttable}
\usepackage{array}
\usepackage{longtable}
\usepackage{rotating}
\usepackage{tabularx}
\usepackage{lastpage}
\usepackage{lscape}
\usepackage{stackengine}
\usepackage{amsmath} 
\usepackage{amsfonts}
\usepackage[T1]{fontenc}
\usepackage{amsthm}
\usepackage{geometry}
\usepackage{graphicx}
\usepackage{bbm}
\usepackage{moreverb}
\usepackage{color, pdfcolmk}
\usepackage{fancyhdr}
\usepackage{csquotes}  
\usepackage{etoolbox}
\makeatletter
\def\@footnotecolor{blue}
\define@key{Hyp}{footnotecolor}{%
 \HyColor@HyperrefColor{#1}\@footnotecolor%
}
\patchcmd{\@footnotemark}{\hyper@linkstart{link}}{\hyper@linkstart{footnote}}{}{}
\makeatother
\hypersetup{footnotecolor=blue}
\theoremstyle{plain}% default
\newtheorem{thm}{Theorem}[section]    % theorem
\newtheorem{prop}[thm]{Proposition}   % proposition
\newtheorem{defn}{Definition}% definition
\usepackage[round]{natbib}
\usepackage[english]{babel}
\usepackage{pgfplots}
\pgfplotsset{compat=1.18}
\usetikzlibrary{arrows.meta,decorations.pathreplacing,positioning}

\title{Bridging Econometrics and AI: VaR Estimation via  Reinforcement Learning and GARCH Models}

\author[1]{Fredy POKOU \thanks{\texttt{fredypokou@gmx.fr}}}
\author[2]{Jules SADEFO KAMDEM \thanks{\texttt{jules.sadefo-kamdem@umontpellier.fr}}}
\author[3]{François BENHMAD \thanks{\texttt{francois.benhmad@umontpellier.fr}}}

\affil[1]{Inria, CNRS, Univ. of Lille, Centrale Lille, UMR 9189 - CRIStAL, F-59000 Lille, France}

\affil[2,3]{MRE UR 209 and Faculty of Economics. Montpellier University, France}

\providecommand{\keywords}[1]{\textbf{\textit{Keywords:}} #1}

\begin{document}

\maketitle
%In an environment of ever-changing financial conditions, risk measurement plays a central role in modern risk management. The growing complexity of financial data calls into question the relevance of the traditional assumptions on which econometric models are based, making their results less reliable and difficult to interpret. 
% 
%To address this issue, we propose a robust empirical approach to Value-at-Risk (VaR) estimation. It is based on the VaR-GARCH model (and its extension VaR-GJR-GARCH) and incorporates an innovative component: the use of a directional market forecast. To improve the accuracy of this forecast, we reformulate the problem in terms of unbalanced classification, solved using the Double Deep Q-Network (DDQN) algorithm, an advanced deep reinforcement learning model. 
%
%The effectiveness of our approach is assessed on daily Eurostoxx 50 data, covering different periods of crisis and major shocks. In addition to conventional statistical tests, this analysis validates the robustness of the model. The results show that the integration of DDQN significantly improves the accuracy of VaR estimates by dynamically adjusting risk levels: a decrease when risks are low and an increase during periods of high volatility. This performance underscores the potential of this methodology to refine risk management and offer better protection against market fluctuations.

\begin{abstract}
Estimating market risk in volatile environments remains a major challenge, as traditional GARCH-type models often struggle to capture nonlinear dynamics. This paper proposes a hybrid Value-at-Risk (VaR) framework that integrates GARCH volatility forecasts with a Double Deep Q-Network (DDQN) reinforcement learning classifier. By reframing VaR estimation as a classification problem, the model adaptively adjusts risk thresholds based on predicted low- and high-risk return regimes.
Using more than 16 years of Euro Stoxx 50 data, the framework achieves 79.4\% test accuracy and substantially reduces both the frequency and the temporal clustering of VaR violations. Backtesting confirms compliance with the Kupiec and Christoffersen tests, while Extreme Value Theory supports its ability to model tail risk. The resulting approach offers a statistically robust, capital-efficient, and regulatory-aligned solution for proactive financial risk management.
\end{abstract}

\keywords{Deep reinforcement learning, Directional prediction, Imbalanced class problem, Value-at-Risk.}

%--------------------------%
%      Introduction        %
%--------------------------%
\section{Introduction}
\label{sec1}
\paragraph{Context:}
Forecasting stock returns is a long-standing challenge in financial economics, with significant implications for both risk management and regulatory compliance. Traditional econometric models such as GARCH \citep{bollerslev1986generalized} capture volatility persistence but fail to fully account for key stylized facts of financial time series: fat tails, volatility clustering, and leverage effects \citep{glosten1993relation}. Similarly, modern machine learning and deep learning methods, although capable of modeling nonlinear dynamics \citep{goodfellow2016deep,tealab2018time}, tend to underperform during rare but impactful market shocks \citep{fawcett1997adaptive,pokou2022contribution}. As illustrated in Figure \ref{fig1}, these limitations often result in systematic mispredictions of excess returns, especially in turbulent markets.
These forecasting inaccuracies are critical because they directly translate into unreliable estimates of Value-at-Risk (VaR), the benchmark risk measure under Basel regulatory frameworks \citep{basel2017basel}. Overestimation inflates capital requirements, whereas underestimation exposes institutions to excessive losses. To mitigate these shortcomings, the recent literature has shifted from precise return forecasting to directional return prediction, reframe the task as a classification problem, determining whether returns will be positive or negative \citep{kanas2001neural,nyberg2011forecasting,alostad2017directional}. Beyond the standard zero threshold, quantile and volatility-based criteria have been introduced to better isolate significant market movements \citep{chung2007model,linton2007quantilogram}. 
\newpage
\noindent Building on this line of research, we follow \cite{nevasalmi2020forecasting} in adopting a finer classification that filters noise and emphasizes extreme outcomes. Although this approach improves the alignment between forecasts and risk exposure, it introduces a severe class imbalance problem, a largely unexplored challenge in the financial risk literature.
\begin{figure}[H]
\begin{subfigure}{.5\textwidth}
  \centering
  % include first image
  \includegraphics[width=.95\linewidth]{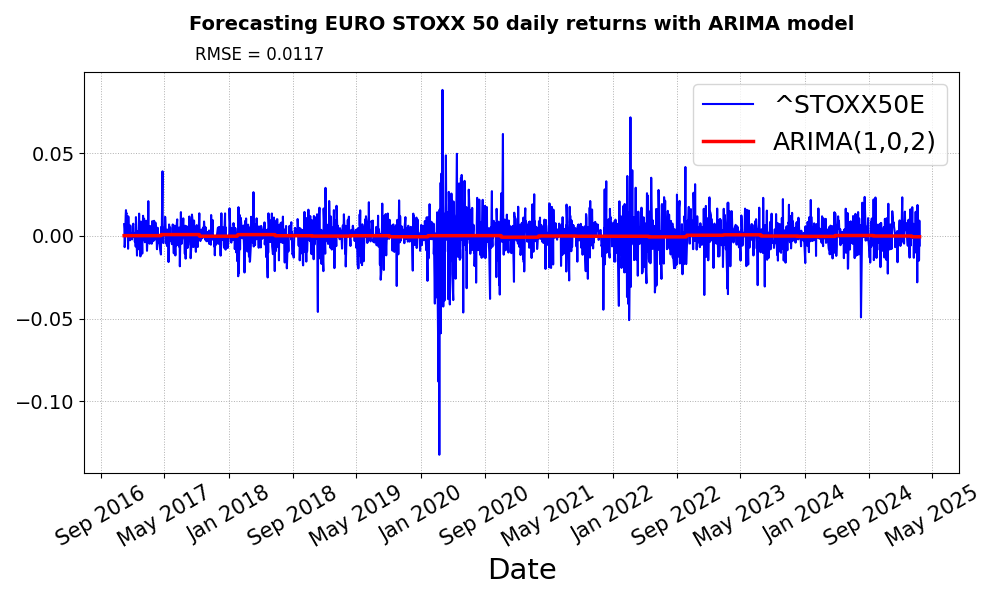}  
  \caption{ARIMA}
  \label{fig1_1}
\end{subfigure}
\begin{subfigure}{.5\textwidth}
  \centering
  % include second image
  \includegraphics[width=.95\linewidth]{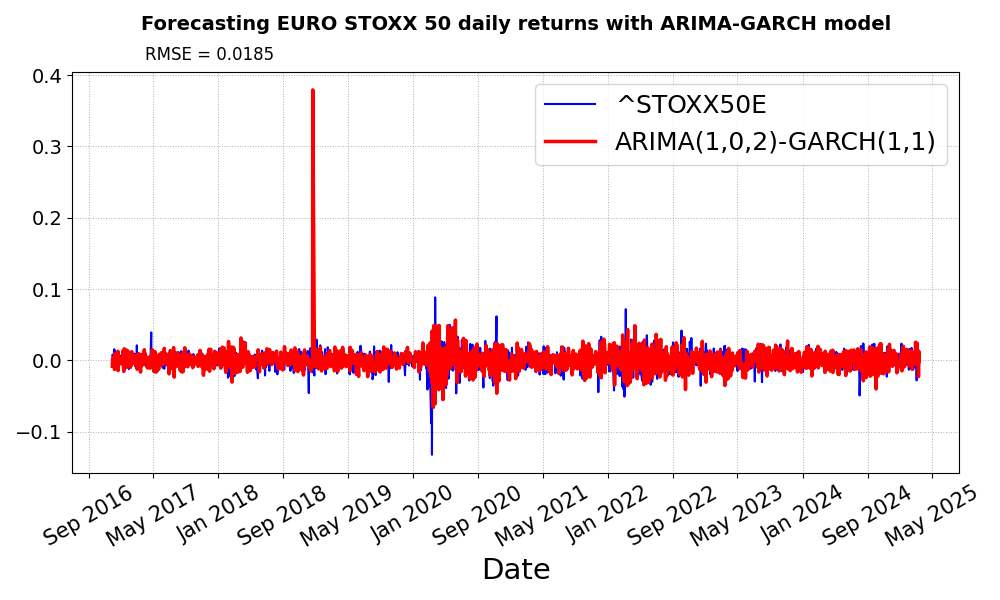}  
  \caption{GARCH}
  \label{fig1_2}
\end{subfigure}
\vfill
\begin{subfigure}{.5\textwidth}
  \centering
  % include third image
  \includegraphics[width=.95\linewidth]{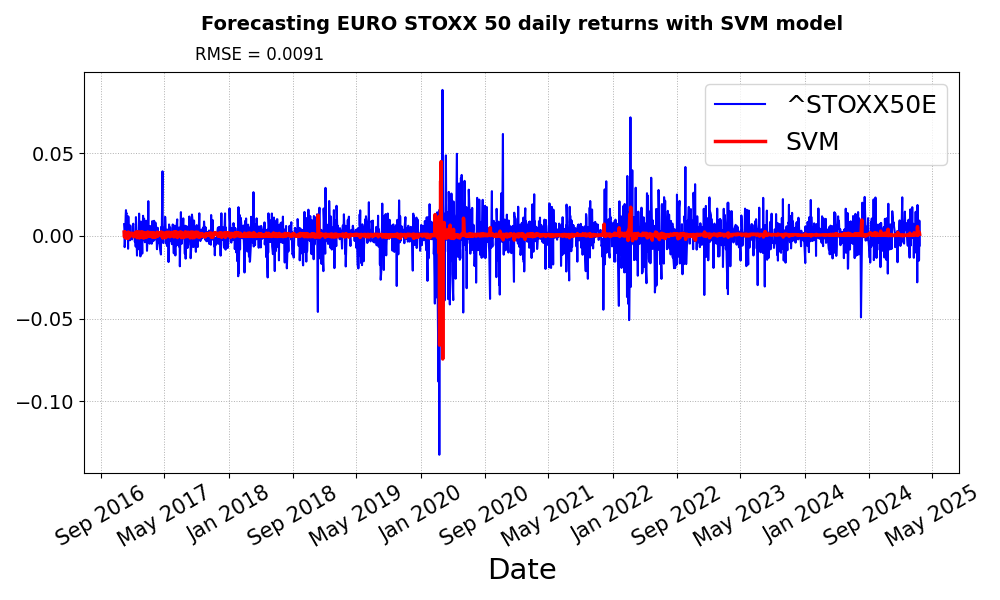}  
  \caption{SVM}
  \label{fig1_3}
\end{subfigure}
\begin{subfigure}{.5\textwidth}
  \centering
  % include fourth image
  \includegraphics[width=.95\linewidth]{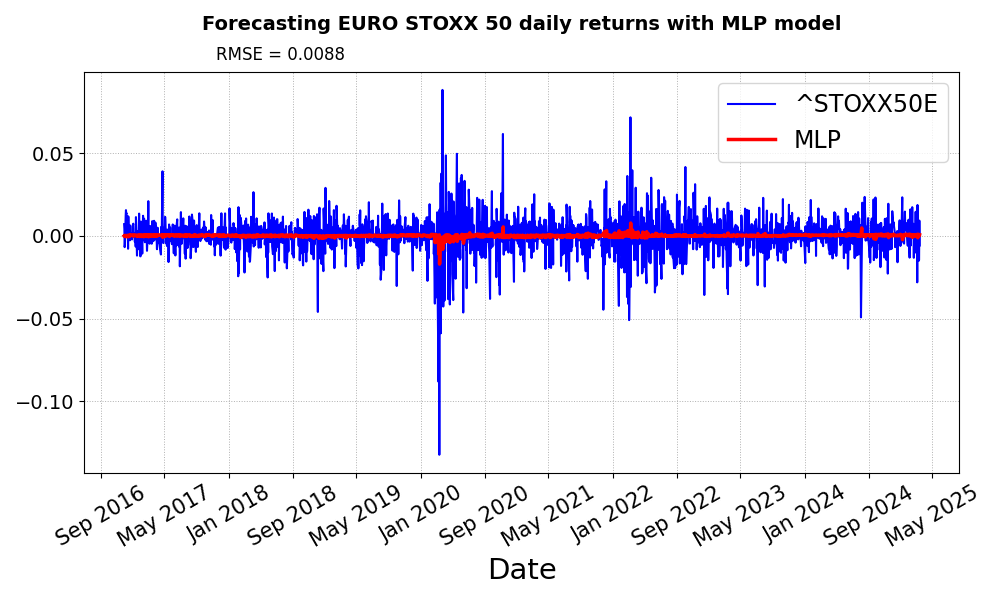}  
  \caption{MLP}
  \label{fig1_4}
\end{subfigure}
\caption{Euro Stoxx 50 return forecasts using different models}
\label{fig1}
\end{figure}

\paragraph{Problem Statement:}
The central challenge of this study is the reliable estimation of the Value-at-Risk (VaR) when the return prediction is reformulated as a directional prediction task under volatile market conditions. In practice, extreme negative returns the outcomes most critical for risk management and regulatory compliance occur far less frequently than small or positive returns. This structural asymmetry, known as the imbalanced class problem, undermines the predictive capacity of both econometric models and supervised machine learning techniques, often leading to systematic underestimation of VaR and, consequently, potential regulatory breaches.
To address this imbalance, researchers have developed various data rebalancing strategies. Among oversampling methods, the widely used Synthetic Minority Oversampling Technique (SMOTE) \citep{chawla2002smote} creates synthetic minority observations to enhance classifier sensitivity, while Borderline-SMOTE \citep{han2005borderline} focuses on generating synthetic data near decision boundaries. On the undersampling side, Tomek Links \citep{tomek1976two,elhassan2016classification} remove borderline majority instances, whereas Edited Nearest Neighbors (ENN) \citep{wilson2007asymptotic,elhassan2016classification} eliminate misclassified majority samples to improve class separability. More recent hybrid techniques combine oversampling and undersampling, for example, SMOTE with Tomek Links to improve balance while reducing noise \citep{zeng2016effective,pereira2020mltl,tang2008svms}. Furthermore, cost-sensitive learning penalizes errors in minority outcomes more heavily, encouraging models to better capture rare but impactful returns \citep{tang2008svms}.
Despite their contributions, these methods often introduce bias or overfitting, particularly in financial contexts where return distributions evolve rapidly and exhibit strong nonlinear dynamics. As a result, conventional rebalancing remains insufficient for adaptive risk forecasting. Against this backdrop, Deep Reinforcement Learning (DRL) offers a promising alternative. Unlike static rebalancing techniques, DRL learns an adaptive policy that continuously adjusts to market conditions while explicitly addressing class imbalance \citep{mnih2013playing,van2016deep,yang2020deep}. 
\newpage
\noindent
%By reframing the prediction of the return of the stock from a regression problem into a directional prediction framework under class imbalance, our approach aims to deliver more robust and regulatory relevant VaR estimates.
By reformulating stock return forecasting, which was previously a regression problem, into a directional forecasting framework in a context of class imbalance, our approach aims to provide more robust and regulatory-relevant VaR estimates.
\paragraph{Main Contributions:}
This paper advances the literature on financial risk modeling by proposing a novel hybrid framework for Value-at-Risk (VaR) estimation that explicitly addresses the challenges of volatility dynamics and class imbalance in directional prediction.\\

\noindent First, based on the recent work of \cite{nevasalmi2020forecasting}, we reformulate return prediction as a directional prediction problem, adopting refined thresholds that better isolate extreme returns and filter out noise. This formulation improves the model’s sensitivity to rare but economically significant market movements, a dimension often neglected by conventional approaches.
Second, we develop a hybrid econometric AI methodology that combines the volatility modeling strengths of the GARCH family with the adaptive capabilities of Deep Reinforcement Learning (DRL). Unlike static machine learning classifiers, DRL enables dynamic policy adjustment, thus directly addressing the imbalanced class problem and improving robustness in turbulent market environments.
Third, we highlight the regulatory and economic relevance of our approach. By reducing both the frequency of VaR violations and the associated capital requirements, the framework provides a risk management tool that is better aligned with the prudential standards of Basel III, mitigating the twin risks of overcapitalization and underestimation of losses.
Finally, we present a rigorous empirical validation using daily Eurostoxx 50 data that cover periods of increased volatility and financial crisis. The results demonstrate that our approach outperforms traditional econometric and supervised learning models in terms of accuracy, stability, and robustness, offering a more reliable foundation for proactive risk management.     

\paragraph{Reviewed Literature: }
Value-at-risk (VaR) estimation has long relied on econometric approaches such as the GARCH family \citep{bollerslev1986generalized, glosten1993relation}, which, despite their analytical rigor, often fail to capture heavy tails, volatility clustering, and leverage effects. To address these limitations, machine learning methods, including neural networks and ensemble models, have been applied to financial time series \citep{zhang2003time,tealab2018time,pokou2024hybridization}. Although such models improve non-linear predictive capacity, they remain sensitive to class imbalance and tend to underestimate rare but impactful shocks.\\

In recent years, risk-sensitive reinforcement learning (RL) has attracted considerable attention as a means of embedding robustness into decision making under uncertainty. \cite{morimura2012parametric} introduced parametric return density estimation through a risk-sensitive Bellman framework, while \cite{ying2022towards} proposed a CVaR-constrained policy optimization algorithm that ensures safety by explicitly controlling downside risk. Extending this line, \cite{noorani2025risk} formulated a risk-sensitive actor-critical method using exponential criteria to improve robustness and sample efficiency. Similarly, \cite{zhang2024cvar} developed CVaR-Constrained Policy Optimization (CVaR-CPO) to ensure safety in reinforcement learning by focusing on the upper tail of cost distributions. In the context of asset allocation, \cite{cui2023portfolio} employed a CVaR-based deep reinforcement learning approach for cryptocurrency portfolios, demonstrating the superiority of DRL in environments with large tail risks.\\

From a distributional perspective, \cite{stanko2019risk} advanced CVaR Q-learning, reducing computational complexity and highlighting the relevance of distributional RL for risk-averse policy design.
Although these contributions significantly advance risk-sensitive RL and CVaR-based methodologies, they primarily focus on direct quantile estimation, risk-sensitive optimization, or portfolio allocation problems. In contrast, our study introduces a fundamentally different perspective by reformulating the VaR estimation as a directional forecasting problem. Inspired by \cite{nevasalmi2020forecasting}, we adopt refined thresholds to isolate extreme returns and explicitly address the resulting class imbalance through Deep Reinforcement Learning. Combined with GARCH volatility modeling, our hybrid framework offers a more adaptive and regulatory-relevant approach, combining econometric rigor with the flexibility of DRL in a way that has not yet been explored in the literature.
\newpage
\paragraph{Paper Organization:}
The remainder of the paper is structured as follows. Section \ref{sec2} introduces the methodological framework, including the GARCH modeling and the DRL architecture for adaptive VaR estimation. Sections \ref{sec3} and \ref{sec4} present the empirical validation of Eurostoxx 50 daily data, discussing robustness checks and comparative analysis. Section \ref{sec5} concludes with the main findings, policy implications, and avenues for future research.
%-------------------------%
%      Methodology        %
%-------------------------%
\section{Methodology}
\label{sec2}
Let $P_t$ be the closing price of an asset at time $t$, so the logarithmic return at time $t$ over a unit time interval is defined as $r_t=\log(P_t/P_{t-1})$.
\subsection{Value-at-Risk}
\label{sec_2_1}
The Value-at-Risk (VaR) at confidence level $\alpha$ is defined as the $\alpha$-quantile of the return distribution and represents the loss that will be exceeded in $\alpha \times 100\%$ of cases over the next period. Formally,
\begin{equation}
P\big(r_{t+1} < -VaR_{t+1}^{\alpha}\big) = \alpha,
\label{e1}
\end{equation}
where $\alpha \in [0,1]$ is typically set to 0.05 or 0.01 in risk management applications. Denoting by $F_{r_{t+1}}(\cdot)$ the cumulative distribution function (CDF) of returns, the VaR can equivalently be expressed as
\begin{equation}
VaR_{t+1}(\alpha) = -F^{-1}_{r_{t+1}}(\alpha),
\label{e2}
\end{equation}
where $F^{-1}_{r_{t+1}}(\alpha)$ is the inverse CDF of a continuous probability distribution, or quantile function, of $r_{t+1}$.\\

Although nonparametric estimation of $F^{-1}_{r_{t+1}}$ is possible, parametric VaR remains widely used in practice due to its simplicity and its ability to incorporate stylized facts of financial returns, such as heavy tails, volatility clustering, and asymmetry \citep{glosten1993relation}. A common starting point is the conditional location scale representation.
\begin{equation}
r_{t+1} = \mu_{t +1} + \sigma_{t+1} z_{t+1},
\label{e3}
\end{equation}
where $\mu_{t+1}$ and $\sigma_{t+1}$ are the conditional mean and volatility, and $z_{t+1}$ is an i.i.d. innovation with zero mean and unit variance. Under this specification, the one-step-ahead VaR at level $\alpha$ takes the unified form
\begin{equation}
VaR_{t+1}(\alpha) = -\mu_{t+1} + \sigma_{t+1} F^{-1}_{z_{t+1}}(\alpha),
\label{e4}
\end{equation}
where $F^{-1}_{z_{t+1}}(\alpha)$ denotes the $\alpha$-quantile of the standardized innovation distribution. For innovations in standard normal distribution, $F^{-1}_{r_{t+1}}(\alpha) = \Phi^{-1}(\alpha)$, while for heavy-tailed data, a standard Student $t$ distribution with $\nu$ degrees of freedom is often preferred, resulting in $F^{-1}_{z_{t+1}}(\alpha) = t_{\nu}^{-1}(\alpha)$.\\

In practice, the conditional mean $\mu_{t+1}$ of daily returns is close to zero and is commonly modeled as a constant. However, conditional volatility $\sigma_{t+1}$ is dynamic and is typically estimated using generalized autoregressive conditional heteroskedasticity models, such as GARCH(p,q) \citep{bollerslev1986generalized} or its asymmetric extensions, including GJR-GARCH \citep{glosten1993relation}. These models capture the persistence and clustering of volatility, while the choice of distribution for $z_t$ accounts for the fat tails frequently observed in financial data.\\

\subsection{GARCH models}
\label{sec_2_2}
Introduced by \cite{bollerslev1986generalized}, the GARCH(p,q) model can be defined by assuming that $\{r_t\}_{t \in \mathbb{N}}$ is a stochastic process decomposable in terms of expected returns and residuals, as follows:
\begin{equation}
r_t= \mu_t + \epsilon_t
\nonumber
\end{equation}
where
\begin{equation}
\begin{cases}
\epsilon_t = \sigma_t z_t, \\[6pt]
\sigma_t^2 = \alpha_0 + \sum_{i=1}^p \alpha_i \epsilon_{t-i}^2 + \sum_{j=1}^q \beta_j \sigma_{t-j}^2,
\end{cases}
\label{e6}
\end{equation}
where $z_t$ is an innovation standardized by i.i.d. with zero mean and unit variance. The positivity of the conditional variance is ensured by $\alpha_0 > 0$, $\alpha_i \geq 0$, and $\beta_j \geq 0$, while weak stationarity requires $\sum_{i=1}^p \alpha_i + \sum_{j=1}^q \beta_j < 1$. \\

The GARCH (1,1) model is the simplest and most commonly used univariate model and can be defined as follows:
\begin{equation}
\left\{
    \begin{array}{ll}
        \epsilon_t  &= \sigma_t z_t  \\
        \sigma_t^2 &= \alpha_0 + \alpha_1 \epsilon_{t-1}^2 + \beta_1 \sigma_{t-1}^{2}
    \end{array}
\right.
\label{e7}
\end{equation}
where the variance process $\sigma_t^2$ in GARCH(1,1) is stationary when $\alpha_1+\beta_1<1$.\\
The convergence rate depends on $\alpha_1+\beta_1$. 
As persistence is closer to 1, it takes longer for a shock to be forgotten by the market.\\

\citet*{glosten1993relation} proposed the GJR-GARCH(p,q) model to overcome GARCH's inability to differentiate between negative and positive shocks and to take leverage into account. 
It can be defined as follows:
\begin{equation}
\left\{
    \begin{array}{ll}
        \epsilon_t  &= \sigma_t z_t  \\
        \sigma_t^2 &= \alpha_0 +\sum_{i=1}^{p} (\alpha_i + \gamma_i \mathds{1}_{\epsilon_{t-i}< 0}) \times \epsilon_{t-i}^2 + \sum_{j=1}^{q}\beta_j \sigma_{t-j}^{2}
    \end{array}
\right.
\label{e8}
\end{equation}
In particular, GJR-GARCH(1,1) is expressed as follows:
\begin{equation}
\left\{
    \begin{array}{ll}
        \epsilon_t  &= \sigma_t z_t  \\
        \sigma_t^2 &= \alpha_0 + (\alpha_1 + \gamma \mathds{1}_{\epsilon_{t-1}< 0}) \times \epsilon_{t-1}^2 + \beta_1 \sigma_{t-1}^{2}
    \end{array}
\right.
\label{e9}
\end{equation}
where $z_t$ is an innovation standardized by i.i.d. and $\mathds{1}_{\{\epsilon_{t-1}<0\}}$ is an indicator function that equals one if the lagged innovation is negative (“bad news”) and zero otherwise. This formulation allows negative shocks to exert a disproportionately larger effect on conditional variance when $\gamma>0$, thus modeling the well-documented leverage effect. The possibility of conditional variance is guaranteed when $\alpha_0>0$, $\alpha_1\geq0$, $\beta_1\geq0$, and $\alpha_1+\gamma\geq0$. The stationarity of the variance process is further ensured under the condition $\alpha_1 + \beta_1 + \tfrac{1}{2}\gamma < 1$. This specification thus provides a parsimonious yet effective framework for accommodating asymmetries in volatility dynamics, a key feature of financial return series.\\

Although there are numerous extensions of the GARCH framework, such as EGARCH, TGARCH, and APARCH, our analysis focuses on the standard GARCH and its asymmetric variant, GJR-GARCH. These models capture the essential stylized facts of financial returns, namely volatility clustering, persistence, heavy tails, and leverage effects. Furthermore, restricting attention to GARCH and GJR-GARCH ensures parsimony and avoids the risk of overparameterization, which is particularly relevant in high-frequency estimation and reinforcement learning contexts. This balance between explanatory power and tractability makes them a robust baseline for integrating econometric volatility modeling with adaptive reinforcement learning techniques.

\subsection{Threshold selection for directional prediction}
\label{sec_2_3}
In the context of directional forecasting, the choice of threshold $c$ is critical for transforming continuous returns into binary outcomes. Although simple strategies often adopt $c=0$ to distinguish between positive and negative returns \citep{kanas2001neural, nyberg2011forecasting}, our objective to improve Value-at-Risk (VaR) estimation requires a more risk-sensitive criterion. Specifically, we define $c$ as the most severe threshold associated with observed VaR violations over a given horizon $H$. Formally, let $k$ index the time periods within the horizon. 
\newpage
Then,
\begin{equation}
c = \max \left\{\, r_{k+1}\;\big|\; r_{k+1}<\mathrm{VaR}_{k+1}(\alpha)\,, \; k=1,\ldots,H \right\}.
\label{e9a}
\end{equation}
Since $\text{VaR}_{k+1}(\alpha)$ is typically negative, this definition ensures that $c$ corresponds to the least negative return among the set of violations, i.e. the'mildest' loss still affecting the VaR limit. The binary outcome variable is then constructed as
\begin{equation}
y_t(c) = \mathds{1}_{\{r_t \leq c\}},
\label{e9b}
\end{equation}
which flags return less than or equal to $c$ as high-risk outcomes.\\ This formulation explicitly links the classification threshold to the realized downside risk, aligning the directional prediction with the practical objective of reducing VaR violations.

\begin{figure}[H]
\begin{subfigure}{.5\textwidth}
  \centering
  % include first image
  \includegraphics[width=.95\linewidth]{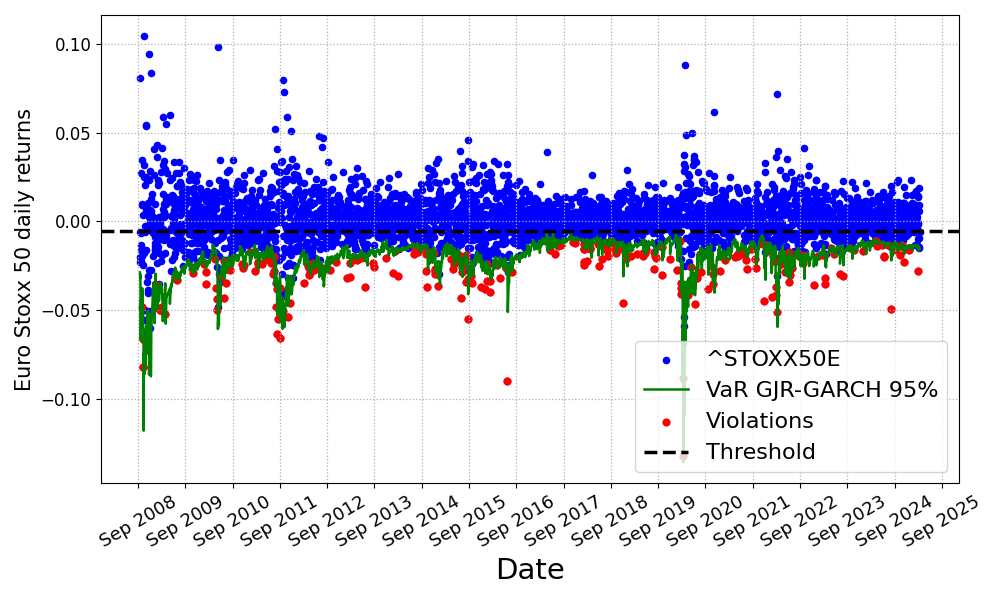}  
  \caption{Threshold selection}
  \label{fig2_a}
\end{subfigure}
\begin{subfigure}{.5\textwidth}
  \centering
  % include second image
  \includegraphics[width=.95\linewidth]{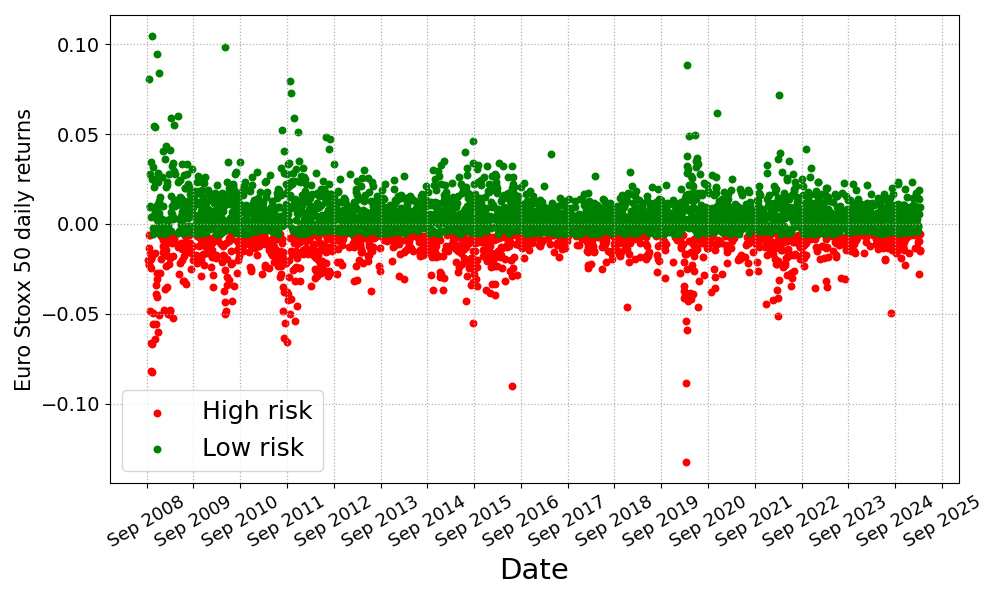}  
  \caption{The two return classes}
  \label{fig2_b}
\end{subfigure}
\caption{Illustration characterization of threshold and return classes}
\label{fig2}
\end{figure}
Figure \ref{fig2} illustrates our threshold-based approach for the directional prediction of returns. The panel \ref{fig2_a} depicts the dynamic selection of the risk threshold $c$, defined as the least severe loss among recent VaR violations, ensuring that the classification criterion is directly anchored in the realized downside risk. This adaptive procedure allows the threshold to evolve with market conditions, reflecting the gravity of observed losses. The panel \ref{fig2_b} translates this threshold into a binary classification of returns: outcomes at or below $c$ are labeled high-risk ($y_t(c)=1$), while all others are labeled low-risk ($y_t(c)=0$). These panels demonstrate how linking threshold selection to actual VaR breaches provides a more sensitive and operationally relevant framework to distinguish between high- and low-risk periods, thus enhancing the predictive capacity of risk management models.

\subsection{Proposed VaR model}
\label{sec_2_4}

Let $\mathrm{VaR}_{t+1}(\alpha)$ denotes the advance value at risk in one step at the confidence level $\alpha$, estimated by a GARCH or GJR-GARCH process. We introduce a classification model that, at each time $t+1$, assigns the portfolio to one of two states:
\begin{itemize}
\item Low risk: corresponding to a lower level of downside exposure;
\item High risk: corresponding to a heightened level of downside exposure.
\end{itemize}
We then define an adjusted risk measure as follows.

\begin{defn}[Classification-Adjusted Value at Risk]
The Classification-Adjusted Value-at-Risk denoted $\mathrm{VaR}_{ML}(\alpha)$, is given by

\begin{equation}
\small
\mathrm{VaR}_{\text{ML}}(\alpha) =
\begin{cases}
(1 - b_1) \cdot \mathrm{VaR}_{t+1}(\alpha), & \text{if } \hat{r}_{t+1} = \text{“Low risk”}, \\[6pt]
(1 + b_2) \cdot \mathrm{VaR}_{t+1}(\alpha), & \text{if } \hat{r}_{t+1} = \text{“High risk”},
\end{cases}
\label{e11_bis}
\end{equation}
where $b_1,b_2 \in \mathbb{R}^+$ are calibration parameters controlling the degree of adjustment, and $\hat{r}_{t+1}$ denotes the predicted risk category.
\end{defn}
\newpage
\begin{prop}
The Classification-Adjusted Value-at-Risk, $\mathrm{VaR}_{ML}(\alpha)$, preserves the fundamental properties of the original Value-at-Risk.
\end{prop}
\begin{proof}
Recall that the Value-at-Risk is positively homogeneous, that is, for any $\lambda \geq 0$,
$$
\mathrm{VaR}_{t+1}^{\lambda X}(\alpha) = \lambda \cdot \mathrm{VaR}_{t+1}^{X}(\alpha).
$$
In our framework, the adjustment factor applied to $\mathrm{VaR}_{t+1}(\alpha)$ is either $(1 - b_1)$ or $(1 + b_2)$, depending on the classification outcome. Since $\hat{r}_{t+1}$ takes a single value in {\text{Low risk}, \text{High risk}}, only one branch of the definition is active at any given time, ensuring that
$$
\mathrm{VaR}_{ML}(\alpha) = \kappa \cdot \mathrm{VaR}_{t+1}(\alpha),
$$
with $\kappa \in \{1-b_1,\,1+b_2\}$. By the positive homogeneity of the original VaR, $\mathrm{VaR}_{ML}(\alpha)$ therefore inherits the same property. Moreover, since the adjustment is monotonic and preserves the sign of $\mathrm{VaR}_{t+1}(\alpha)$, subadditivity and translation invariance are unaffected. Hence, $\mathrm{VaR}_{ML}(\alpha)$ retains the theoretical properties of the original Value-at-Risk.
\end{proof}

\begin{figure}[h!]
    \centering
    \includegraphics[width=0.5\linewidth]{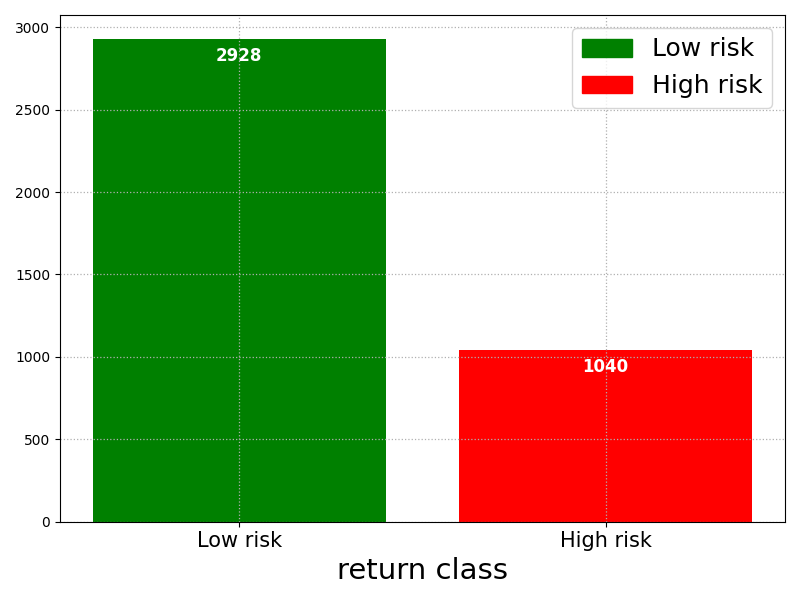}
    \caption{Proportion of different return classes}
    \label{fig3}
\end{figure}

\subsection{Machine Learning}
Figure \ref{fig3} illustrates the imbalanced class problem arising from the threshold-based binary classification of returns presented in Figure \ref{fig2_b}. In this setting, the majority of returns are classified as low-risk, while the minority corresponding to high-risk events linked to Value-at-Risk (VaR) violations remains underrepresented. Traditional supervised learning models such as logistic regression \citep{dreiseitl2002logistic}, Support Vector Machines \citep{vapnik1999nature,tang2008svms} and neural network architectures including perceptrons, multilayer perceptrons \citep{rosenblatt1958perceptron,murtagh1991multilayer}, recurrent neural networks \citep{ampomah2020evaluation,sunny2020deep}, and temporal convolutional networks \citep{lea2016temporal}, tend to favor the majority class, thus reducing their effectiveness in detecting rare but crucial high-risk outcomes.\\

Several sampling strategies have been proposed to address this imbalance. Undersampling methods such as Edited Nearest Neighbors (ENN), Random Undersampling, and Tomek Links \citep{kotsiantis2006handling,barandela2004imbalanced,elhassan2016classification,zeng2016effective,pereira2020mltl} reduce the dominance of the majority class, but at the cost of discarding potentially informative observations, thereby impairing model robustness. In contrast, oversampling techniques such as SMOTE, Borderline-SMOTE, ADASYN, and Random Oversampling \citep{chawla2002smote,han2005borderline,zhang2014rwo} artificially enrich the minority class, but risk overfitting, since synthetic instances may not capture the true underlying distribution.
\newpage 
\noindent Moreover, standard optimization objectives like mean squared error (MSE) or cross-entropy loss are dominated by the majority class, leading to weight updates that minimize overall error without improving minority class detection.\\

To evaluate models under such an imbalance, we employ a comprehensive set of performance metrics. Accuracy, precision, recall, and the F1 score, each capturing different trade-offs between type I and type II errors. However, these metrics remain sensitive to class imbalance, often overstating performance when the majority class dominates. To mitigate this bias, we additionally adopt the Geometric Mean (G-Mean), defined as
$$
G\text{-}Mean = \sqrt{\text{Recall} \times \text{Specificity}},
$$
which balances sensitivity to both minority-class detection (Recall) and majority-class recognition (Specificity). Unlike the F1-score, which emphasizes precision-recall trade-offs, the G-Mean penalizes classifiers that perform well in one class while neglecting the other, making it particularly suited to highly imbalanced financial data.\\

Given the shortcomings of sampling techniques and conventional metrics, we adopt the Double Deep Q-Network (DDQN) \citep{van2016deep} as a reinforcement learning–based alternative. Unlike purely supervised models, DDQN learns adaptive policies through interaction with the environment, optimizing a reward function rather than a static loss. This enables a more effective handling of class imbalance by dynamically focusing on rare but high impact outcomes, offering a more robust framework for risk-sensitive classification and ultimately improving the accuracy of VaR estimation.

\subsection{Reinforcement learning} 
Although supervised learning relies on labeled data to approximate a mapping between inputs and outputs, its performance deteriorates in settings characterized by severe class imbalance or sequential dependence. Reinforcement learning (RL) provides a fundamentally different paradigm by seeking to learn an optimal policy that maximizes cumulative rewards through interaction with a dynamic environment. Rather than minimizing prediction error, the agent operates in a trial-and-error framework: it selects actions, observes outcomes, and updates its strategy based on rewards, enabling adaptive learning in contexts where rules are uncertain or evolve over time.\\

Originally formalized by \cite{sutton1998reinforcement}, RL extends beyond the limitations of dynamic programming \citep{puterman2014markov}, which requires complete knowledge of the environment and suffers from the 'curse of dimensionality' and beyond Monte Carlo approaches \citep{gilks1995markov}, which fail to exploit temporal dependencies effectively. The standard framework is that of a Markov Decision Process (MDP) \citep{white1989markov}, in which states, actions, and rewards are linked by probabilistic transitions. Central to this approach is the value function, which estimates expected long-term returns and is recursively defined by the Bellman equation. The agent–environment interaction underlying this MDP structure is illustrated in Figure \ref{fig4}, which highlights the sequential feedback loop that drives policy optimization.\\

The integration of deep learning into RL has led to powerful algorithms such as the Double Deep Q-Network (DDQN) \citep{van2016deep}  capable of scaling to high-dimensional and nonstationary environments. Using this framework, RL offers a flexible and robust alternative to supervised models, particularly in financial contexts characterized by volatility, structural breaks, and class imbalance. Its adaptive decision-making capacity makes it especially suitable for Value-at-Risk estimation, where capturing rare but high-impact events is essential for effective risk management.

\begin{figure}[H]
\centering
\includegraphics[scale=0.7]{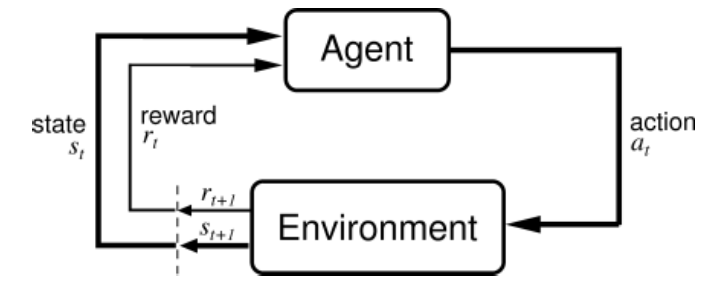}
\caption{Agent-environment interaction model \citep{sutton1998reinforcement}}
\label{fig4}
\end{figure}
\noindent Building on these principles, we adopt the Double Deep Q-Network (DDQN) framework, which combines the flexibility of reinforcement learning with the representational power of deep neural networks, providing an adaptive and robust approach to addressing the class imbalance and dynamic risk conditions inherent in Value-at-Risk estimation.

\subsubsection{Markov Decision Process}
\label{sec2_6_1}
To address the challenge of Value-at-Risk estimation under class imbalance, we formulate the problem as a Markov Decision Process (MDP). An MDP is defined by the tuple $\langle \mathcal{S}, \mathbb{A}, \mathcal{P}, \mathcal{R}, \gamma \rangle$, where each element corresponds to a key component of the sequential decision-making framework.
\begin{itemize}
\item States ($\mathcal{S}$): Each state $s_t \in \mathcal{S}$ represents the financial environment at time $t$, including the information needed to classify the return (e.g., recent volatility dynamics and lagged returns).

\item Actions ($\mathbb{A}$): The agent chooses between two possible actions: $a_t \in {0,1}$, where $0$ indicates a prediction of 'low risk' (majority class) and $1$ indicates 'high risk' (minority class).

\item State Transitions ($\mathcal{P}$): The transition function $\mathcal{P}$ governs how the environment evolves after an action. In our context, transitions are deterministic in the sense that the next state $s_{t+1}$ is determined by the sequence of observed returns, while the chosen action influences only the reward signal.

\item Rewards ($\mathcal{R}$): The reward function evaluates the quality of the agent’s classification and follows a design widely used in the reinforcement learning literature for unbalanced classification tasks \citep{firdous2023imbalanced}. In particular, the correct detection of rare but critical 'high-risk' events (true positives) is strongly rewarded, while missing such an event (false negative) incurs the most severe penalty. To correct for class imbalance, we incorporate a scaling factor $\rho \in (0,1]$, defined as the ratio of minority to majority class sizes. This ensures that both classes contribute proportionally to the agent’s learning process, a principle consistent with cost-sensitive reinforcement learning frameworks.
The reward $R_{t+1}$ can be defined as :
\begin{equation}
R_{t+1} = 
\left\{
    \begin{array}{ll}
        + \rho &  \mbox{$a_t$ is True Negative}  \\
        -\rho  &  \mbox{$a_t$ is False Negative}\\
        + 1    &  \mbox{$a_t$ is True Positive} \\
        -1     &  \mbox{$a_t$ is False Positive}
    \end{array}
\right.
\label{eq;;eq12x}
\end{equation}
Concretely:
\begin{itemize}[label=$\square$]
  \item True Negative:  $+\rho$ (moderate reward for correct classification of the majority class)
  \item False Negative: $-\rho$ (penalty weighted to reflect the rarity and severity of high-risk events)
  \item True Positive:  $+1$ (successful detection of a critical event)
  \item False Positive: $-1$ (penalty to discourage unnecessary false alarms).
\end{itemize}
\item Discount Factor ($\gamma$): Finally, $\gamma \in [0,1]$ balances immediate and future rewards, ensuring that the agent optimizes long-term performance rather than short-sighted performance.
\end{itemize}
%
%This formulation allows the agent to treat VaR estimation as a sequential decision-making problem. By reinforcing correct predictions of rare but critical high-risk events and penalizing costly misclassifications, the MDP framework provides a principled way to handle class imbalance while adapting dynamically to evolving financial conditions.

\subsubsection{Policy and Value Functions}
\label{sec2_6_2}

In reinforcement learning, a policy $\pi: \mathcal{S} \rightarrow \mathbb{A}$ defines the decision rule mapping each state $s_t \in \mathcal{S}$ to a probability distribution over possible actions. In our setting, the policy can be viewed as a classifier, where actions correspond to risk labels: $a \in \{0,1\}$, with 0 denoting 'low risk' and 1 denoting 'high risk.' Formally,
\begin{equation}
\pi_t(a|s) = \mathbb{P}(a_t=a \,|\, s_t=s).
\label{eq;;eq13x}
\end{equation}
The agent’s goal is to maximize the expected cumulative discounted reward, defined as:
\begin{equation}
G_t = \sum_{k=0}^{\infty} \gamma^{k} R_{t+k},
\label{eq;;eq14x}
\end{equation}
where $\gamma \in [0,1)$ is the discount factor that balances the importance of immediate versus future rewards.\\
The state value function in policy $\pi$ measures the expected return starting from state $s$ and then following policy $\pi$:
\begin{equation}
v_{\pi}(s) = \mathbb{E}_{\pi}[G_t \,|\, s_t=s].
\label{eq;;eq15x}
\end{equation}
The optimal state-value function is then
\begin{equation}
v^{*}(s) = \max_{\pi} v_{\pi}(s), \quad \forall s \in \mathcal{S}.
\label{eq;;eq16x}
\end{equation}
Similarly, the action value function (or Q function) evaluates the expected return of taking action $a$ in state $s$ and thereafter following policy $\pi$:
\begin{equation}
Q_{\pi}(s,a) = \mathbb{E}_{\pi}[G_t \,|\, s_t=s, a_t=a].
\label{eq;;eq17x}
\end{equation}
The optimal Q-function is
\begin{equation}
Q^{*}(s,a) = \max_{\pi} Q_{\pi}(s,a), \quad \forall s \in \mathcal{S}, \, a \in \mathbb{A}.
\label{eq;;eq18x}
\end{equation}
Using Bellman’s optimality principle, the recursive form is expressed as follows:
\begin{equation}
Q^{*}(s,a) = \mathbb{E}\left[R_t + \gamma \max_{a'} Q^{*}(s_{t+1}, a') \,\middle|\, s_t=s, a_t=a \right].
\label{eq;;eq19x}
\end{equation}
This formalization, consistent with the applications in finance \citep{cui2023portfolio}, allows the policy to dynamically adapt to signals. The agent gradually learns to classify high-risk states more accurately by associating state–action pairs with long-term returns, overcoming the limitations of static supervised learning in highly imbalanced financial data.

\subsubsection{Double Deep Q-Network (DDQN)}
\label{sec2_6_3}

Traditional Q-learning \citep{watkins1989learning} estimates the value of each state–action pair through iterative updates, but its reliance on tabular Q-tables makes it impractical in environments with large or continuous state spaces. 
The transition from Q-learning to Deep Q-Networks (DQN) addresses the curse of dimensionality inherent in tabular approaches. In DQN, the Q-function is approximated by a deep neural network parameterized by $\theta$, yielding an estimate of the value of taking action $a_t$ in state $s_t$. The target value combines the immediate reward with the discounted future return:
\begin{equation}
y_t^{DQN} = R_{t+1} + \gamma \max_{a_{t+1}\in \mathbb{A}} \hat{Q}_{\pi}(s_{t+1},a_{t+1};\theta^{-}),
\label{eq;;eq20x}
\end{equation}
where $R_{t+1}$ reflects the reward associated with the classification at time $t+1$, $\gamma$ is the discount factor, and $\theta^{-}$ denotes the weights of the target network, updated periodically to stabilize learning. The network parameters are optimized by minimizing the loss function:
\begin{equation}
\mathcal{L}^{DQN}(\theta) = \mathbb{E}\left[ \big(y_t^{DQN} - Q_{\pi}(s_t,a_t;\theta)\big)^{2} \right].
\label{eq;;eq21x}
\end{equation}
This loss represents the expected square difference between the predicted Q-value and the target $y_t^{DQN}$. In our classification context, minimizing $\mathcal{L}^{DQN}(\theta)$ ensures that the model better aligns its predictions (classifying 'Low risk' vs. 'High risk') with the true risk labels, especially rewarding correct detection of minority ('high-risk') cases. The replay buffer supplies mini-batches of past experiences, enabling balanced exposure to both classes despite data imbalance.
Despite these improvements, standard DQNs often suffer from systematic overestimation of Q-values, leading to suboptimal policy selection. To mitigate this, the Double Deep Q-Network (DDQN) introduces two distinct networks: an evaluation network that selects actions and a target network that evaluates them. The update rule can be expressed compactly as
\begin{equation}
Q_{\pi}^{(i)}(s_t,a_t) \;\leftarrow\; (1-\alpha) Q_{\pi}^{(i)}(s_t,a_t) 
+ \alpha \Big[ R_{t+1} + \gamma \max_{a_{t+1}\in \mathbb{A}} Q_{\pi}^{(j)}(s_{t+1},a_{t+1}) \Big],
\label{eq;;eq22x}
\end{equation}
with $i \in \{1,2\}$, $j=3-i$, and $\alpha$ the learning rate. Here, one network determines which action (classification) appears optimal, while the other provides the corresponding Q-value, reducing the upward bias of single-network estimation. The DDQN target value is then
\begin{equation}
y_t^{DDQN} = R_{t+1} + \gamma \max_{a_{t+1}\in \mathbb{A}} \hat{Q}_{\pi}\!\left(s_{t+1},\arg\max_{a_{t+1}\in \mathbb{A}} Q_{\pi}(s_{t+1},a_{t+1};\theta);\theta^{-}\right).
\label{eq;;eq23x}
\end{equation}
Training proceeds by minimizing the corresponding DDQN loss function,
\begin{equation}
\mathcal{L}^{DDQN}(\theta) = \mathbb{E}\!\left[ \big(y_t^{DDQN} - Q_{\pi}(s_t,a_t;\theta)\big)^{2} \right].
\label{eq;;eq24x}
\end{equation}
In practice, this loss penalizes incorrect risk predictions more heavily for the minority class ('high risk') due to the reward structure, ensuring that the model gradually prioritizes the accurate classification of rare but critical events. By combining experience replay with the alternating update of evaluation and target networks, DDQN achieves both stability and robustness, directly improving the detection of Value-at-Risk violations.

%-------------------------%
%    Empirical results    %
%-------------------------%
\section{Empirical results}
\label{sec3}
\subsection{Data and Features Selection}
We evaluated our models using the Euro Stoxx 50 index, which represents the 50 largest capitalizations in the eurozone. The dataset, obtained from Yahoo\footnote{\url{https://fr.finance.yahoo.com/}} Finance, covers the period from September 1, 2008, to March 13, 2025. Although the Euro Stoxx 50 is a natural benchmark for European equity markets, it also presents important challenges for predictive modeling. Because the index undergoes periodic revisions, with constituents added or removed to reflect evolving market conditions, the direct use of its lagged components introduces discontinuities and structural breaks. This dynamic would require continuous recalibration of models, reducing both their robustness and their generalizability.\\

To address this issue, we adopt a more stable framework that emphasizes variables that capture broad market and macroeconomic dynamics rather than firm-specific fluctuations. The explanatory set therefore integrates major eurozone equity indices such as the CAC 40 and DAX, which provide information on regional and sectoral diversification. Key currency pairs including EUR/USD and EUR/GBP are incorporated to reflect the role of exchange rates in European competitiveness and capital flows. In addition, technical indicators such as Simple and Exponential Moving Averages and the Relative Strength Index are included, as they translate price dynamics into signals that can anticipate short-term market reversals. Finally, conditional volatility forecasts derived from GARCH and GJR-GARCH specifications are employed to account for the heteroskedastic and asymmetric nature of financial returns.

\newpage

\begin{figure}[h!]
    \centering
    \includegraphics[width=0.7\linewidth]{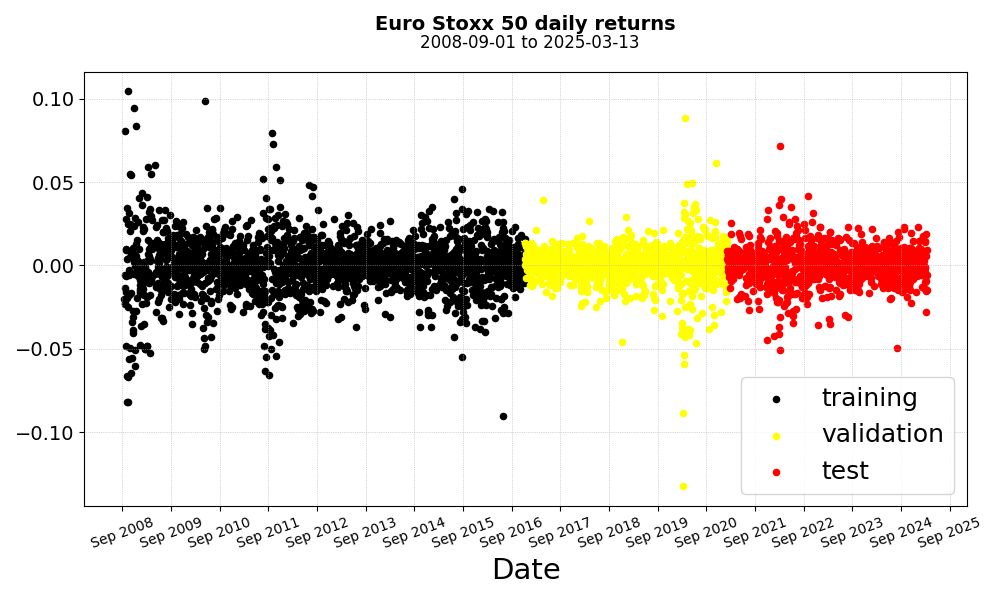}
    \caption{Data set splitting}
    \label{fig5}
\end{figure}
To ensure a rigorous and unbiased evaluation of the predictive framework, the data set is chronologically divided into three distinct sub-samples: training, validation, and test (Figure \ref{fig5}). The training sample, which constitutes the bulk of the data, is used to estimate the model parameters by learning the relationships between the explanatory variables and the target risk classification. The validation set serves to fine-tune the hyperparameters and prevent overfitting, thereby enhancing the generalizability of the model. The final test sample is reserved exclusively for out-of-sample evaluation, providing a realistic assessment of predictive performance on unseen data and ensuring that the reported results are not driven by in-sample optimization.

\begin{figure}[H]
    \centering
    \includegraphics[width=0.9\linewidth]{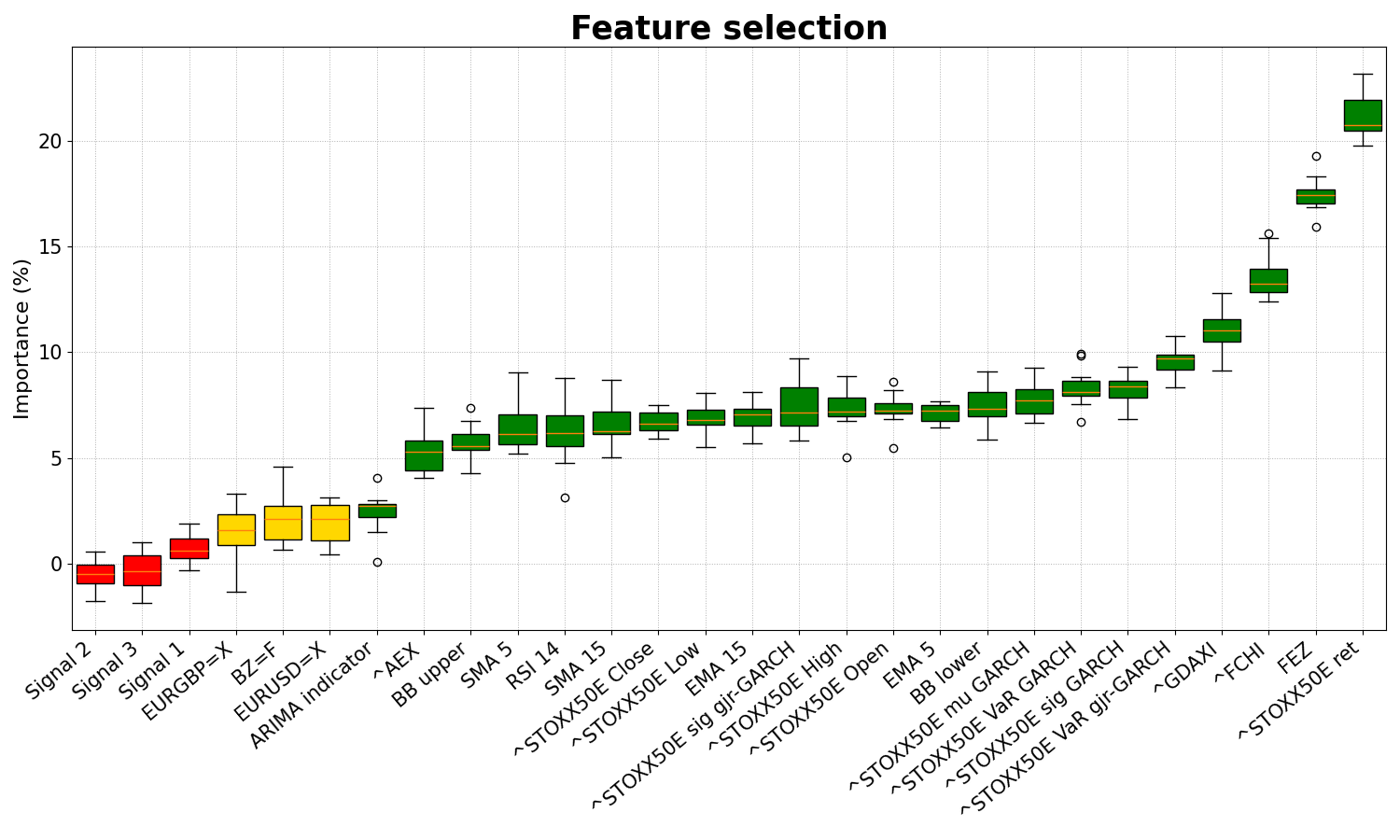}
    \caption{Features selection}
    \label{fig6}
\end{figure}

The selection of features is performed using the Boruta algorithm \citep{kursa2010feature,kursa2010boruta}, a wrapper method particularly suited for high-dimensional nonlinear financial datasets. This approach identifies explanatory variables with a statistically significant contribution to predictive accuracy. Figure \ref{fig6} presents the ranking of the selected features. Economically, the results are intuitive: The lagged returns of the Euro Stoxx 50 account for roughly 20\% of the explanatory power, consistent with the notion that the past information is rapidly incorporated by the European equity markets. The CAC 40 contributes approximately 13\%, reflecting the importance of sectoral and regional effects. The volatility forecasts explain the difference between 8\% and 10\%, confirming the central role of expected volatility in risk formation. Interestingly, the FEZ ETF, which closely tracks the Euro Stoxx 50, adds an additional 17\%, suggesting that liquidity effects and intraday adjustments convey information not fully embedded in the raw index. In contrast, exchange rates such as EUR/USD and EUR/GBP display weaker explanatory power, consistent with their indirect influence on daily equity risk. Technical indicators also provide additional, though more limited, information, particularly in detecting momentum shifts.\\

The descriptive statistics for the returns, explanatory characteristics, strategy positions, and qualitative variables are reported in Tables \ref{tab1}, \ref{tab2}, \ref{tab3} and \ref{tab4}. Table \ref{tab1} confirms that the return distributions are stationary but non-normal, exhibiting skewness and excess kurtosis. These findings underscore the prevalence of fat tails and justify the reliance on nonlinear predictive models. While stationarity ensures the validity of econometric inference, the presence of non-normality highlights the need to capture extreme events. Table \ref{tab2} further reveals a significant heterogeneity in the scale of the explanatory variables, motivating the use of Min-Max normalization as a pre-processing step. This transformation improves numerical stability and accelerates convergence in hybrid econometric-machine learning frameworks, according to recommendations in the literature \citep{han2005borderline}. Tables \ref{tab3} and \ref{tab4} provide additional information on the signals generated by trading strategies and the forecasts derived from the ARIMA and hybrid ARIMA-ANN models, both of which enrich the information set for subsequent classification tasks.

\begin{table}[H]
\centering
\resizebox{12.7cm}{!}{
\begin{tabular}{|l|rrrrrrrr|}
\hline
Assets & Mean (\%) & Std (\%) & Skewness & Kurtosis & Min & Max & ADF & Jarque-Bera  \\
\midrule
EURUSD=X & -0.007 & 0.724 & 0.782 & 113.692 & -0.143 & 0.160 & -12.606 & 2136941.801 \\
EURGBP=X & 0.001 & 0.598 & 0.166 & 62.569 & -0.111 & 0.110 & -33.930 & 647122.057 \\
BZ=F & -0.006 & 2.440 & -0.606 & 13.962 & -0.280 & 0.212 & -63.680 & 32465.986 \\
FEZ & 0.019 & 1.741 & -0.409 & 9.130 & -0.133 & 0.162 & -11.813 & 13887.933 \\
$^{\wedge}$FCHI & 0.017 & 1.409 & -0.188 & 8.671 & -0.131 & 0.106 & -21.735 & 12451.577 \\
$^{\wedge}$GDAXI & 0.034 & 1.399 & -0.087 & 8.742 & -0.131 & 0.116 & -20.348 & 12637.285 \\
$^{\wedge}$AEX & 0.022 & 1.311 & -0.327 & 9.672 & -0.114 & 0.100 & -12.401 & 15532.029 \\
$^{\wedge}$STOXX50E & 0.014 & 1.429 & -0.245 & 7.909 & -0.132 & 0.104 & -21.293 & 10379.093 \\
\hline
\end{tabular} 
}
\caption{Descriptive statistics for returns} 
\label{tab1}
\end{table}
\begin{table}[H]
\centering
\resizebox{10.cm}{!}{
\begin{tabular}{|l|rrrr|}
\hline
Features & Mean  & Std  & Min & Max \\
\midrule
$^{\wedge}$STOXX50E Close & 3346.722 & 706.846 & 1809.980 & 5540.690 \\
$^{\wedge}$STOXX50E High & 3370.447 & 705.616 & 1823.250 & 5568.190 \\
$^{\wedge}$STOXX50E Low & 3321.543 & 707.570 & 1765.490 & 5506.570 \\
$^{\wedge}$STOXX50E Open & 3346.747 & 706.032 & 1812.780 & 5531.530 \\
\hline
BB upper & 3469.825 & 694.591 & 2139.143 & 5597.252 \\
BB lower & 3213.444 & 701.508 & 1670.766 & 5324.596 \\
EMA 5 & 3345.631 & 703.870 & 1854.835 & 5495.911 \\
EMA 15 & 3342.931 & 696.996 & 1954.528 & 5455.968 \\
SMA 5 & 3345.639 & 704.382 & 1857.532 & 5490.536 \\
SMA 15 & 3342.891 & 698.354 & 1925.779 & 5481.875 \\
RSI 14 & 52.626 & 11.333 & 10.728 & 79.894 \\
\hline
$^{\wedge}$STOXX50E $\mu$ GARCH & -0.000 & 0.002 & -0.049 & 0.024 \\
$^{\wedge}$STOXX50E sig GARCH & 0.013 & 0.006 & 0.005 & 0.057 \\
$^{\wedge}$STOXX50E sig gjr-GARCH & 0.013 & 0.006 & 0.005 & 0.057 \\
$^{\wedge}$STOXX50E VaR GARCH & -0.021 & 0.011 & -0.138 & -0.006 \\
$^{\wedge}$STOXX50E VaR gjr-GARCH & -0.021 & 0.011 & -0.138 & -0.006 \\
\hline
\end{tabular}
}
\caption{Descriptive statistics for features} 
\label{tab2}
\end{table}
%
%----------------%
%     table 3    %
%----------------%
\begin{table}[H]
\centering
\resizebox{9cm}{!}{
\begin{tabular}{|l|rrr|}
\hline
Strategy & Sell & Stay & Buy \\ 
\hline
Simple Moving
Average Cross & 3681 & 143 & 143 \\
Exponential Moving Average Cros & 3631 & 168 & 168 \\
Relative Strength Index & 3662 & 165 & 140 \\
\hline
\end{tabular}
}
\caption{Descriptive statistics for positions in different trading strategies} 
\label{tab3}
\end{table}
\newpage
%
%----------------%
%     table 4    %
%----------------%

\begin{table}[H]
\centering
\resizebox{7.cm}{!}{
\begin{tabular}{|l|rr|}
  \hline
 Variables & Low risk & High risk \\ 
  \hline
  ARIMA  indicator  & 2951 &  1016 \\  
   \hline
   \textbf{risk level}  & \textbf{2928} &  \textbf{1039} \\  
   \hline
\end{tabular}
}
\caption{Descriptive statistics of qualitative variables}
\label{tab4} 
\end{table}
\subsection{Classification under Imbalanced Data}
The target variable is the binary risk level, defined in Section \ref{sec_2_3} and visualized in Figure \ref{fig7}. In training, validation and test sets, low-risk observations represent about 70\%, with high-risk accounting for 30\%. Although not extreme, this imbalance poses a challenge: Standard classifiers tend to overexcite the majority class, impairing the detection of high-risk periods, precisely those with the greatest economic costs.
To address this, we compared under-sampling and over-sampling strategies. Undersampling reduced imbalance, but at the cost of lost information. In contrast, the ADASYN algorithm \citep{he2008adasyn}  effectively generated synthetic minority samples, improving the detection of high-risk states with minimal bias, albeit at the cost of a longer training.

\begin{figure}[H]
 \centering
 \begin{subfigure}[b]{0.3\textwidth}
     \centering
     \includegraphics[width=\textwidth]{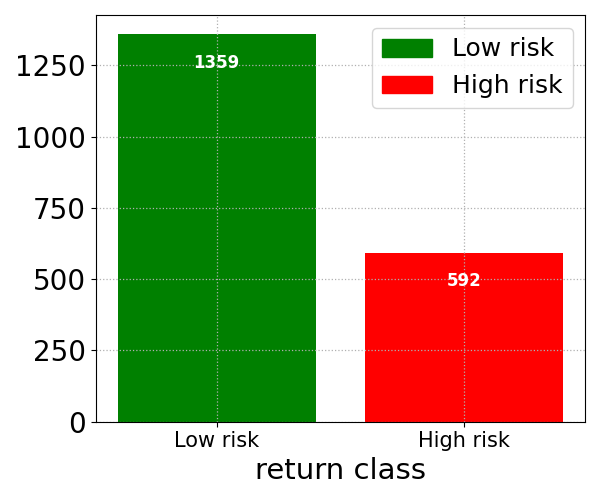}
     \caption{Over training period}
  \label{fig7_1}
 \end{subfigure}
 \hfill
 \begin{subfigure}[b]{0.3\textwidth}
     \centering
     \includegraphics[width=\textwidth]{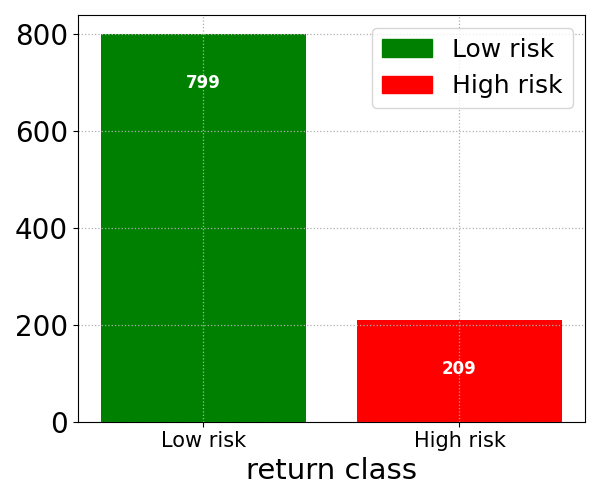}
     \caption{Over validation period}
  \label{fig7_2}
 \end{subfigure}
 \hfill
 \begin{subfigure}[b]{0.3\textwidth}
     \centering
     \includegraphics[width=\textwidth]{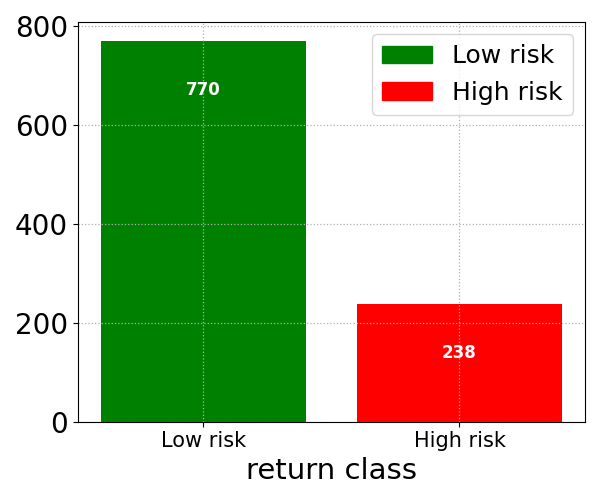}
     \caption{Over testing period}
     \label{fig7_3}
 \end{subfigure}
    \caption{Barplot of different risk levels}
    \label{fig7}
\end{figure}

We benchmarked multiple models: Logistic regression (LR), Support Vector Machines (SVM), Single-Layer Perceptron (ANN), Multi-Layer Perceptron (MLP), Temporal Convolutional Network (TCN), and our Double Deep Q-Network (DDQN). All supervised methods were combined with ADASYN for fairness.
Model performance is assessed not only by accuracy, but also by F1 score, recall, precision, and G-mean, as accuracy alone may be misleading in imbalanced contexts. Economically, these metrics align with the risk manager’s priorities: Recall reflects the ability to capture high-risk episodes; precision limits costly false alarms; and G-Mean balances the two, avoiding asymmetric errors.
Table \ref{tab5} reports the results. Reinforcement Learning achieved the strongest performance: accuracies of 83.5\% (validation) and 79.4\% (test), F1 scores greater than 0.66 and recall values of 0.574 (validation) and 0.54 (test), the highest among all models. G-Mean scores (0.735 validation, 0.714 test) confirm balanced detection between classes. TCN also performed well, but consistently behind RL.\\

The economic importance is clear: the superior recall of the RL framework minimizes undetected high-risk returns, reducing the likelihood of underestimated losses, and enabling more proactive risk mitigation strategies.

%
%----------------%
%     table 5    %
%----------------%
\begin{table}[H]
\begin{subtable}[c]{0.5\textwidth}
\centering
\resizebox{8.2cm}{!}{
\begin{tabular}{|l|rrrrrr|}
  \hline
 Metric & LR & SVM & ANN & MLP & TCN & RL \\ 
  \hline
Accuracy & 0.616 & 0.639 & 0.618 & 0.719 & 0.788 & 0.835 \\
F1-score & 0.375 & 0.370 & 0.368 & 0.499 & 0.587 & 0.668 \\
Recall & 0.283 & 0.290 & 0.280 & 0.396 & 0.492 & 0.574 \\
Precision & 0.555 & 0.512 & 0.536 & 0.675 & 0.727 & 0.799 \\
G-Mean & 0.489 & 0.494 & 0.485 & 0.596 & 0.672 & 0.735 \\
   \hline
\end{tabular}
}
\label{tab5_a}
\subcaption{On validation sample}
\end{subtable}
\begin{subtable}[c]{0.5\textwidth}
\centering
\resizebox{8.2cm}{!}{
\begin{tabular}{|l|rrrrrr|}
  \hline
 Metric & LR & SVM & ANN & MLP & TCN & RL \\ 
  \hline
Accuracy & 0.514 & 0.524 & 0.530 & 0.678 & 0.720 & 0.794 \\
F1-score & 0.412 & 0.431 & 0.398 & 0.527 & 0.595 & 0.661 \\
Recall & 0.289 & 0.300 & 0.285 & 0.403 & 0.452 & 0.540 \\
Precision & 0.723 & 0.765 & 0.660 & 0.761 & 0.870 & 0.853 \\
G-Mean & 0.492 & 0.508 & 0.485 & 0.602 & 0.653 & 0.714 \\ 
   \hline
\end{tabular}
}
\label{tab5_b}
\subcaption{On testing sample}
\end{subtable}
\caption{Model performance}
\label{tab5}
\end{table}

\newpage
\begin{figure}[H]
    \centering
    \includegraphics[width=0.8\linewidth]{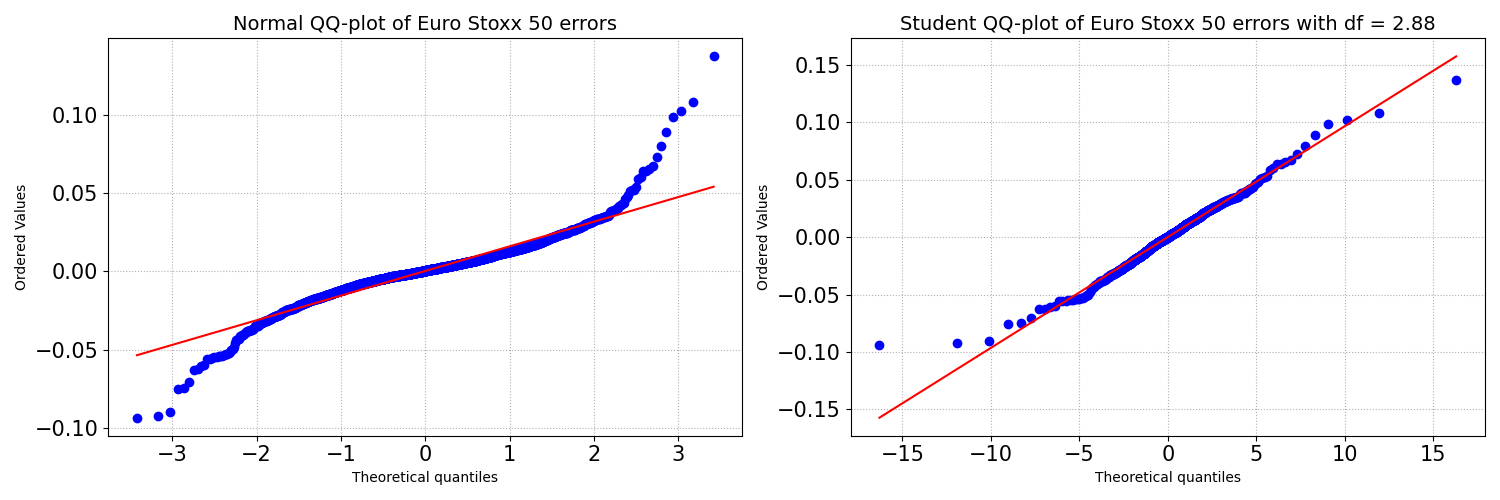}
    \caption{QQ-plot of Euro Stoxx 50 errors distribution}
    \label{fig8}
\end{figure}

\subsection{Evaluation of the Classification-Adjusted Value-at-Risk}

After determining the accuracy of the risk classification, we now evaluate the Classification Adjusted Value-at-Risk (VaR) framework defined in Equation \eqref{e11_bis}. This approach scales traditional GARCH or GJR-GARCH VaR estimates up or down depending on predicted risk level, using adjustment parameters $b_1$ and $b_2$.
To capture conditional volatility dynamics, we use both the GARCH and GJR-GARCH specifications. Given empirical evidence of fat tails (Table \ref{tab1}, Figure \ref{fig8}), we adopt the Student distribution $t$, which better reflects the heavy-tailed nature of the Euro Stoxx 50 returns.
Tables \ref{tab6}–\ref{tab8} summarize the results. As shown in Table \ref{tab6}, the RL-adjusted VaR significantly reduces skewness and kurtosis in predictive distributions. For example, the $VaR^{\text{GARCH}}_{\text{RL}}(\alpha=5\%)$ model achieves a skewness of -2.048 and a kurtosis of 8.613 in the test sample, far closer to empirical distributions than conventional models.
The results of the backtest confirm the robustness. The Christoffersen conditional coverage test (Table \ref{tab7}) indicates that RL-adjusted VaR models pass through validation and test samples, for both 5\% and 1\% levels, unlike conventional models that are systematically rejected. The Kupiec test (Table \ref{tab8}) further validates the adequacy of the violation frequencies, with the RL models showing no systematic bias.\\

From an economic point of view, these results demonstrate that RL-adjusted VaR not only achieves statistical validity but also enhances risk management decisions. By reducing false alarms (overestimated risk) and missed warnings (underestimated risk), it enables more efficient capital allocation, ensures regulatory compliance, and improves early warning capacity during volatile periods.

%----------------%
%     table 6    %
%----------------%
\begin{table}[H]
\begin{subtable}[c]{0.5\textwidth}
\centering
\resizebox{8.32cm}{!}{
\begin{tabular}{|l|l|rrrrrrr|}
\hline
Model                                      & Version  & Min & Median & Mean & Std & Skewness & Kurtosis & Max \\ 
\hline
                                           & Original & -0.136 & -0.013 & -0.016 & 0.012 & -5.101 & 35.265 & -0.005 \\ 
$\small{VaR^{\text{GARCH}}(\alpha=5\%)}$   & TCN      & -0.129 & -0.010 & -0.014 & 0.012 & -4.747 & 31.630 & -0.004 \\
                                           & RL       & -0.138 & -0.010 & -0.014 & 0.011 & -4.929 & 36.258 & -0.004 \\ 
  \hline
                                           & Original & -0.148 & -0.013 & -0.016 & 0.012 & -5.436 & 40.562 & -0.005 \\ 
$\small{VaR^{\text{GJR}}(\alpha=5\%)}$     & TCN      & -0.141 & -0.010 & -0.014 & 0.012 & -4.891 & 35.097 & -0.004 \\
                                           & RL       & -0.137 & -0.011 & -0.014 & 0.012 & -5.092 & 38.121 & -0.004 \\ 
  \hline
                                           & Original & -0.194 & -0.022 & -0.027 & 0.018 & -4.481 & 27.036 & -0.011 \\ 
$\small{VaR^{\text{GARCH}}(\alpha=1\%)}$   & TCN      & -0.190 & -0.017 & -0.024 & 0.018 & -4.280 & 25.984 & -0.008 \\
                                           & RL       & -0.205 & -0.017 & -0.023 & 0.018 & -4.402 & 29.152 & -0.008 \\
  \hline
                                           & Original & -0.215 & -0.021 & -0.027 & 0.019 & -4.772 & 31.483 & -0.010 \\ 
 $\small{VaR^{\text{GJR}}(\alpha=1\%)}$    & TCN      & -0.210 & -0.017 & -0.024 & 0.019 & -4.350 & 28.127 & -0.007 \\ 
                                           & RL       & -0.204 & -0.018 & -0.023 & 0.019 & -4.534 & 30.670 & -0.007 \\ 
   \hline
\end{tabular}
}
\label{tab6_a}
\subcaption{On validation sample}
\end{subtable}
\begin{subtable}[c]{0.5\textwidth}
\centering
\resizebox{8.32cm}{!}{
\begin{tabular}{|l|l|rrrrrrr|}
\hline
Model & Version & Min & Median & Mean & Std & Skewness & Kurtosis & Max \\ 
\hline
                                           & Original & -0.059 & -0.015 & -0.017 & 0.005 & -2.588 & 13.231 & -0.008 \\ 
$\small{VaR^{\text{GARCH}}(\alpha=5\%)}$   & TCN      & -0.071 & -0.014 & -0.015 & 0.007 & -2.118 & 9.694 & -0.005 \\
                                           & RL       & -0.069 & -0.013 & -0.015 & 0.006 & -2.048 & 8.613 & -0.005 \\
  \hline
                                           & Original & -0.054 & -0.015 & -0.017 & 0.006 & -2.033 & 6.564 & -0.007 \\
$\small{VaR^{\text{GJR}}(\alpha=5\%)}$     & TCN      & -0.065 & -0.014 & -0.016 & 0.008 & -1.759 & 5.142 & -0.005 \\
                                           & RL       & -0.057 & -0.013 & -0.015 & 0.007 & -1.757 & 4.732 & -0.005 \\
  \hline
                                           & Original & -0.097 & -0.026 & -0.028 & 0.008 & -2.485 & 11.885 & -0.014 \\ 
$\small{VaR^{\text{GARCH}}(\alpha=1\%)}$   & TCN      & -0.116 & -0.024 & -0.026 & 0.011 & -2.031 & 8.645 & -0.010 \\ 
                                           & RL       & -0.116 & -0.022 & -0.025 & 0.010 & -2.058 & 8.783 & -0.010 \\ 
  \hline
                                           & Original & -0.083 & -0.026 & -0.028 & 0.009 & -1.997 & 6.015 & -0.013 \\
 $\small{VaR^{\text{GJR}}(\alpha=1\%)}$    & TCN      & -0.100 & -0.023 & -0.026 & 0.013 & -1.723 & 4.754 & -0.010 \\
                                           & RL       & -0.096 & -0.022 & -0.025 & 0.011 & -1.773 & 4.874 & -0.010 \\
   \hline
\end{tabular}
}
\label{tab6_b}
\subcaption{On testing sample}
\end{subtable}
\caption{Statistical indicators for VaR models}
\label{tab6}
\end{table}
%
%----------------%
%     table 7    %
%----------------%
\begin{table}[H]
\begin{subtable}[c]{0.5\textwidth}
\centering
\resizebox{8.5cm}{!}{
\begin{tabular}{|l|l|ccccc|l|}
\hline
Model                         & Version & Expected  & Actual  & Test  & Test  &  Test  & Decision \\
 & & violations &  violations & statistic & critical value &  p-value &  \\
\hline
                              & Original & 50                 & 73  & 9.4259 & 3.8414 & 0.0368 & Reject $H_0$ \\
$VaR^{GARCH}(\alpha=5\%)$     & TCN      & 50                 & 72  & 8.6521 & 3.8414 & 0.0032 & Reject $H_0$ \\
                              & RL       & 50                 & 52  & 0.0075 & 3.8414 & 0.9310 & Accept $H_0$ \\
\hline
                              & Original & 50                 & 68  & 5.8597 & 3.8414 & 0.0154 & Reject $H_0$ \\
$VaR^{GJR}(\alpha=5\%)$ & TCN      & 50                 & 72  & 8.6521 & 3.8414 & 0.0032 & Reject $H_0$  \\
                              & RL       & 50                 & 53  & 0.4234 & 3.8414 & 0.5002 & Accept $H_0$ \\
\hline
                              & Original & 10                  & 19  & 6.3276  & 3.8414 & 0.0118 & Reject $H_0$ \\
$VaR^{GARCH}(\alpha=1\%)$     & TCN      & 10                  & 25  & 13.9951 & 3.8414 & 0.0002 & Reject $H_0$ \\
                              & RL       & 10                  &12 & 0.8288  & 3.8414  & 0.3856 & Accept $H_0$ \\
\hline
                              & Original & 10                  & 19   & 6.3276  & 3.8414 & 0.0118 & Reject $H_0$ \\
$VaR^{GJR}(\alpha=1\%)$ & TCN      & 10                  & 24   & 13.9951 & 3.8414 & 0.0002 & Reject $H_0$\\
                              & RL       & 10                  & 11   & 0.0824 & 3.8414 & 0.7741 & Accept $H_0$\\
\hline
\end{tabular}
}
\label{tab7_a}
\subcaption{On validation sample}
\end{subtable}
\begin{subtable}[c]{0.5\textwidth}
\centering
\resizebox{8.5cm}{!}{
\begin{tabular}{|l|l|ccccc|l|}
\hline
Model                         & Version & Expected  & Actual  & Test  & Test  &  Test  & Decision \\
 & & violations &  violations & statistic & critical value &  p-value &  \\ 
\hline
                              & Original & 50                 & 64  & 3.5722 & 3.8414 & 0.0587 & Accept $H_0$ \\
$VaR^{GARCH}(\alpha=5\%)$     & TCN      & 50                 & 47  & 0.2467 & 3.8414 & 0.6193 & Accept $H_0$ \\
                              & RL       & 50                 & 52  & 0.0075 & 3.8414 & 0.9310 & Accept $H_0$ \\
\hline
                              & Original & 50                 & 64  & 3.5722 & 3.8414 & 0.0587 & Accept $H_0$ \\
$VaR^{GJR}(\alpha=5\%)$ & TCN      & 50                 & 52  & 0.0529 & 3.8414 & 0.8180 & Accept $H_0$  \\
                              & RL       & 50                 & 55  & 0.4297 & 3.8414 & 0.5121 & Accept $H_0$ \\
\hline
                              & Original & 10                  & 13  & 0.7828  & 3.8414  & 0.3762 & Accept $H_0$ \\
$VaR^{GARCH}(\alpha=1\%)$     & TCN      & 10                  & 13  & 0.7828  & 3.8414  & 0.3762 & Accept $H_0$ \\
                              & RL       & 10                  & 11  & 0.0823  & 3.8414  & 0.7741  & Accept $H_0$ \\
\hline
                              & Original & 10                  & 12   & 0.3481  & 3.8414 & 0.5551 & Accept $H_0$ \\
$VaR^{GJR}(\alpha=1\%)$ & TCN      & 10                  & 10   & 0.0006  & 3.8414 & 0.9797 & Accept $H_0$\\
                              & RL       & 10                  & 11   & 0.0824 & 3.8414 & 0.7741 & Accept $H_0$ \\
\hline
\end{tabular}
}
\label{tab7_b}
\subcaption{On testing sample}
\end{subtable}
\caption{Christoffersen test of VaR models}
\label{tab7}
\end{table}

%----------------%
%     table 8    %
%----------------%
\begin{table}[H]
\begin{subtable}[c]{0.5\textwidth}
\centering
\resizebox{8.5cm}{!}{
\begin{tabular}{|l|l|ccc|l|}
\hline
Model                          & Version   & Test statistic & Test critical value &  Test p-value & Decision \\ 
\hline
                               & Original  & 13.4338 & 5.9914 & 0.0012 & Reject $H_0$ \\
$VaR^{GARCH}(\alpha=5\%)$      & TCN       & 10.2577 & 5.9914 & 0.0059 & Reject $H_0$ \\
                               & RL        & 7.6298 & 5.9914  & 0.0320 & Reject $H_0$ \\
\hline
                               & Original  & 5.9509  & 5.9914 & 0.0510 & Accept $H_0$ \\
$VaR^{GJR}(\alpha=5\%)$  & TCN       & 10.2577 & 5.9914 & 0.0059 & Reject $H_0$ \\
                               & RL        & 5.9593  & 5.9914 & 0.0510 & Accept $H_0$ \\
\hline
                               & Original & 7.0584  & 5.9914 & 0.02932 & Reject $H_0$ \\
$VaR^{GARCH}(\alpha=1\%)$      & TCN      & 16.0076 & 5.9914 & 0.0003 & Reject $H_0$ \\
                               & RL       & 4.5624  & 5.9914 & 0.1021 & Accept $H_0$ \\
\hline
                               & Original & 0.9907  & 5.9914 & 0.6093 & Accept $H_0$ \\
$VaR^{GJR}(\alpha=1\%)$  & TCN      & 15.1671 & 5.9914 & 0.0005 & Reject $H_0$ \\
                               & RL       & 3.4969  & 5.9914 & 0.1740 & Accept $H_0$ \\
\hline
\end{tabular}
}
\label{tab8_a}
\subcaption{On validation sample}
\end{subtable}
\begin{subtable}[c]{0.5\textwidth}
\centering
\resizebox{8.5cm}{!}{
\begin{tabular}{|l|l|ccc|l|}
\hline
Model                          & Version   & Test statistic & Test critical value &  Test p-value & Decision \\ 
\hline
                               & Original  & 4.5021 & 5.9914  & 0.1052 & Accept $H_0$ \\
$VaR^{GARCH}(\alpha=5\%)$      & TCN       & 0.2660 & 5.9914  & 0.8754 & Accept $H_0$ \\
                               & RL        & 1.3806 & 5.9914  & 0.5014 & Accept $H_0$ \\
\hline
                               & Original  & 3.5735  & 5.9914 & 0.1675 & Accept $H_0$ \\
$VaR^{GJR}(\alpha=5\%)$  & TCN       & 0.2643  & 5.9914 & 0.8761 & Accept $H_0$ \\
                               & RL        & 2.3948  & 5.9914 & 0.3019 & Accept $H_0$ \\
\hline
                               & Original & 1.1229  & 5.9914 & 0.5703 & Accept $H_0$ \\
$VaR^{GARCH}(\alpha=1\%)$      & TCN      & 1.1229  & 5.9914 & 0.5703 & Accept $H_0$ \\
                               & RL       & 4.5624  & 5.9914 & 0.1021 & Accept $H_0$ \\
\hline
                               & Original & 0.6376  & 5.9914 & 0.7270 & Accept $H_0$ \\
$VaR^{GJR}(\alpha=1\%)$  & TCN      & 0.2012  & 5.9914 & 0.9042 & Accept $H_0$ \\
                               & RL       & 0.3253  & 5.9914 & 0.8498 & Accept $H_0$ \\
\hline
\end{tabular}
}
\label{tab8_b}
\subcaption{On testing sample}
\end{subtable}
\caption{Kupiec test of VaR models}
\label{tab8}
\end{table}
%

%
%-------------------------%
%         Section 4       %
%-------------------------%
\section{Discussions }
\label{sec4}

\subsection{Calibration of Classification-Adjusted Value-at-Risk parameters}

To ensure the statistical credibility and practical relevance of the proposed Classification-Adjusted Value-at-Risk (VaR) model, we calibrate the adjustment parameters $(b_1, b_2)$ using two complementary methodologies. The first relies on a rolling-window cross-validation procedure designed to account for the temporal dependence inherent in financial time series. The second adopts a fully Bayesian perspective through Markov Chain Monte Carlo (MCMC) estimation, allowing a probabilistic characterization of parameter uncertainty.
The cross-validation approach explores a grid of $(b_1, b_2)$ combinations and evaluates their performance in terms of the number of VaR violations across validation windows. As illustrated in Figure \ref{fig:params1}, the violation surface shows a clear minimum, indicating that the pair $(b_1 = 0.30, b_2 = 0.20)$ produces the lowest number of exceedances. This empirical optimum ensures that the adjusted VaR provides sufficient coverage while preserving the independence of violations over time, as further confirmed by serial independence tests. Economically, this adjustment corresponds to reducing VaR by 30\% in periods predicted as low risk and increasing it by 20\% in high risk states. Such magnitudes reflect a well-balanced trade-off: they partially correct for potential underestimation without disproportionately relying on the predictive power of the classification model. To avoid overreactive adjustments, the grid search was limited to values below 0.5.
\begin{figure}[H]
        \centering
        \begin{subfigure}[b]{0.475\textwidth}
            \centering
            \includegraphics[width=1.12\textwidth]{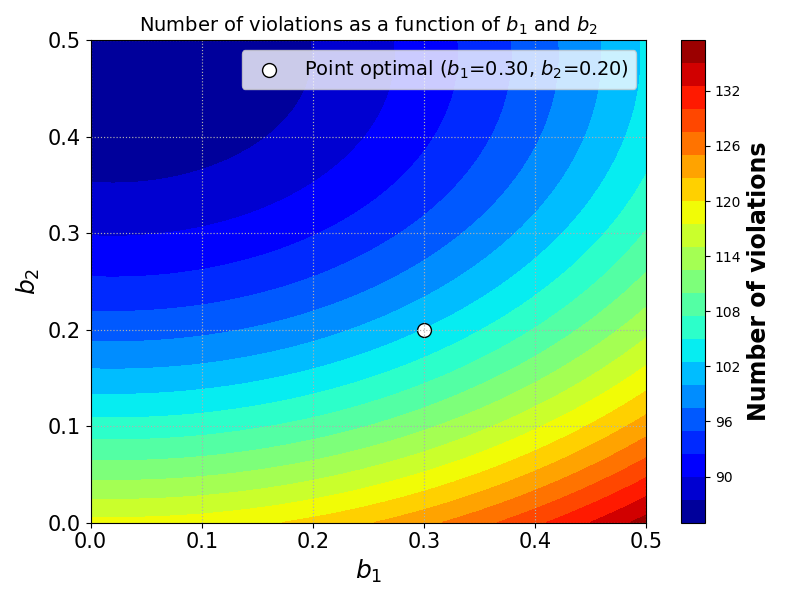}
            \caption[Network2]%
            {{\small Grid search}}    
            \label{fig:params1}
        \end{subfigure}
        \hfill
        \begin{subfigure}[b]{0.475\textwidth}  
            \centering 
            \includegraphics[width=1.1\textwidth]{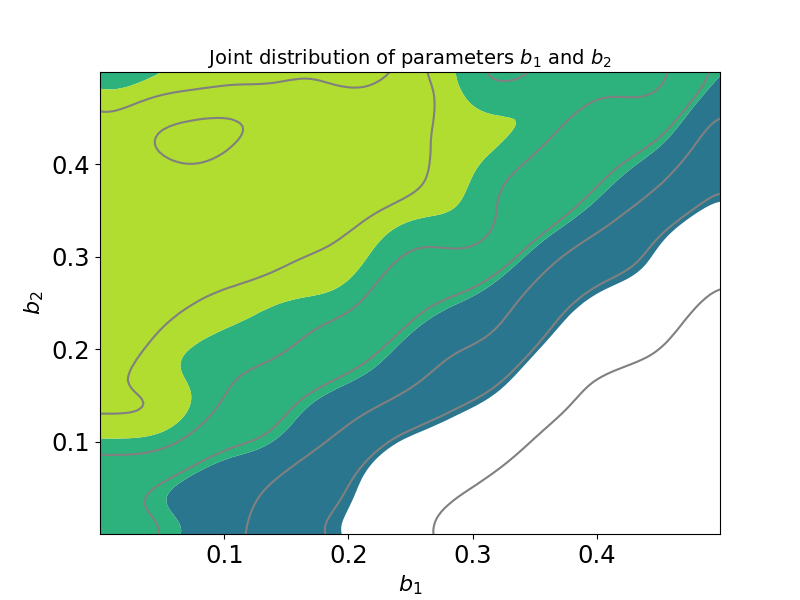}
            \caption[]%
            {{\small  Bayesian approach}}    
            \label{fig:params2}
        \end{subfigure}
        
        \caption
        {\small Calibration of parameters $b_1$ and $b_2$} 
        \label{fig9a}
\end{figure}
To further substantiate the robustness of these parameter choices, we estimate the joint posterior distribution of $(b_1, b_2)$ using a Bayesian MCMC sampling scheme. Figure \ref{fig:params2} shows the resulting posterior density, with darker regions denoting higher credibility. The empirical optimum $(0.30, 0.20)$ lies within the high-density region, reinforcing its plausibility under a statistical inference framework. 
\newpage
Interestingly, the posterior contours reveal a negative correlation between the parameters: stronger downward adjustments in low-risk periods (higher $b_1$) must be counterbalanced by more conservative upward adjustments in high-risk states (lower $b_2$) to maintain consistent coverage. This interdependence highlights the importance of calibrating the adjustment mechanism as a coherent system rather than in isolation. %
Taken together, the cross-validation and Bayesian analyses provide a rigorous and mutually reinforcing justification for the chosen adjustment parameters.\\

Although cross-validation ensures empirical performance in a realistic, time-dependent setting, Bayesian inference offers a probabilistic understanding of parameter uncertainty. This dual calibration strategy improves the credibility, transparency and robustness of the proposed VaR adjustment mechanism, providing a principled foundation for practical deployment in financial risk management.

\subsection{Robustness Analysis via Extreme Value Theory}
Beyond traditional backtesting procedures such as the Kupiec and Christoffersen tests, it is essential to examine the robustness of our Classification Adjusted Value-at-Risk (VaR) model in the tails of the return distribution. In this regard, Extreme Value Theory (EVT) offers a rigorous statistical foundation to evaluate the model's ability to capture rare but financially devastating events. According to the Pickands–Balkema–de Haan theorem, the distribution of losses exceeding a high threshold asymptotically follows a Generalized Pareto Distribution (GPD), making it a natural candidate for tail modeling in financial risk.
To operationalize this framework, we extract the exceedances over the adjusted VaR thresholds $\text{VaR}_\text{{ML}}(\alpha)$ and fit a GPD to these residual losses. The adequacy of the fitted model is assessed using the Kolmogorov–Smirnov (KS) test, where the null hypothesis posits that the observed exceedances are drawn from a GPD. A p-value exceeding 5\% implies that the null hypothesis cannot be rejected, suggesting consistency between empirical tail behavior and the fitted extreme value model.\\

Table \ref{tabEVT} reports the location of the estimated GPD parameters $\hat{\mu}$, scale $\hat{\beta}$, and shape $\hat{\xi}$ alongside the KS test statistics and p values for both validation and testing samples, across the different model configurations (Original, TCN and RL). The results reveal a systematic pattern: For all configurations, the KS p-values consistently exceed the 5\% threshold, confirming the statistical validity of the GPD fit across the board.
A particularly noteworthy observation concerns the validation sample, which chronologically aligns with the Covid-19 crisis period, a phase marked by high volatility, structural breaks, and extreme tail risk.\\

During this period, several RL-adjusted models, such as $\text{VaR}^{\text{GARCH}}_{\text{RL}}(5\%)$ and $\text{VaR}^{\text{GJR}}_{\text{RL}}(1\%)$, produced estimated shape parameters $\hat{\xi} > 1$. This is economically significant: a shape parameter above one implies an infinite second moment of the loss distribution, consistent with the breakdown of conventional risk assumptions in systemic events. Such heavy-tailed behavior reinforces the need to use conservative and adaptive risk measures capable of capturing non-linear tail dynamics.\\

From a risk management point of view, these findings provide strong empirical support for the classification-adjusted VaR framework. Not only do the chosen adjustment parameters $(b_1 = 0.30, b_2 = 0.20)$ ensure accurate coverage in the backtesting, but they also preserve consistency with EVT-based tail modeling. Dual validation in terms of unconditional violations and conditional tail distributions substantiates the robustness of our proposed approach under real-world stress conditions.
Finally, the robustness of the adjusted VaR in both validation and testing samples, as indicated by p-values ranging from 0.06 to 0.99, confirms that the model does not underestimate risk in critical periods. This is crucial for both regulatory compliance (e.g., under Basel III recommendations) and internal risk governance, where reliable tail estimation informs capital adequacy planning and strategic hedging decisions.
\newpage
\subsection{Comparative Performance of VaR Adjustments}

The integration of classification-adjusted Value-at-Risk (VaR) models, particularly those relying on reinforcement learning (RL) and temporal convolutional networks (TCN), leads to marked improvements in both statistical performance and economic efficiency relative to conventional volatility-based models. These gains are most evident when examining the model’s ability to control violation rates, adapt to changing risk regimes, optimize capital allocation, and maintain robustness in structurally unstable market environments.
%
%From a statistical point of view, VaR models based on RL have better properties. As evidenced by the Kupiec and Christoffersen backtest results (Tables \ref{tab7} and \ref{tab8}), the RL approach maintains coverage levels that closely match theoretical expectations while satisfying the independence assumption of violations, a property that is systematically rejected for GARCH and GJR-GARCH models at conventional significance levels. The absence of temporal clustering in violations based on RL model highlights its improved ability to capture the stochastic structure of returns.
From a statistical perspective, Reinforcement Learning–driven VaR models exhibit superior calibration properties. As evidenced by the Kupiec and Christoffersen backtest results (Tables \ref{tab7} and \ref{tab8}), the RL approach maintains coverage levels closely aligned with theoretical expectations while satisfying the independence assumption of violations, a condition systematically rejected for GARCH and GJR-GARCH models at conventional significance levels. The absence of temporal clustering in violations further underscores the improved ability of the RL framework to capture the stochastic structure of returns.
\begin{figure}[H]
        \centering
        \begin{subfigure}[b]{0.475\textwidth}
            \centering
            \includegraphics[width=1.25\textwidth]{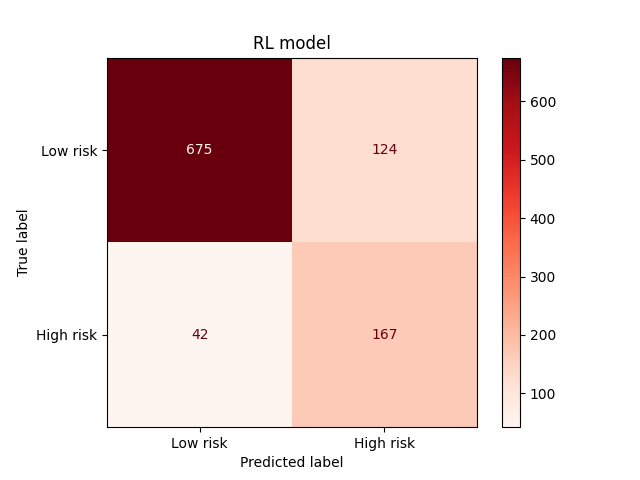}
            \caption[Network2]%
            {{\small On validation sample}}    
            \label{fig9_a}
        \end{subfigure}
        \hfill
        \begin{subfigure}[b]{0.475\textwidth}  
            \centering 
            \includegraphics[width=1.25\textwidth]{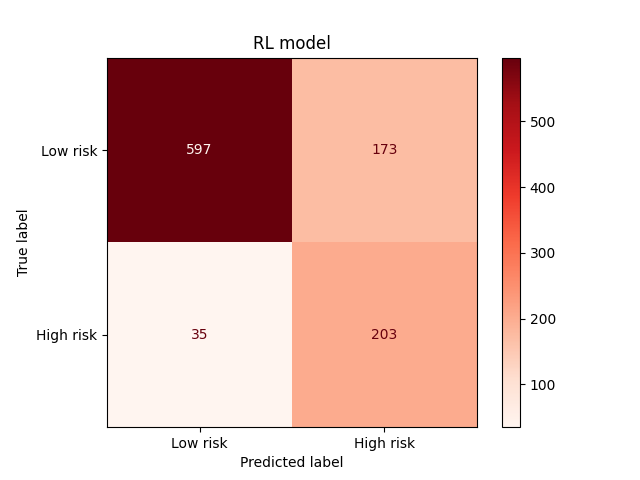}
            \caption[]%
            {{\small On testing sample}}    
            \label{fig9_b}
        \end{subfigure}
        \vskip\baselineskip
        \begin{subfigure}[b]{0.475\textwidth}   
            \centering 
            \includegraphics[width=1.25\textwidth]{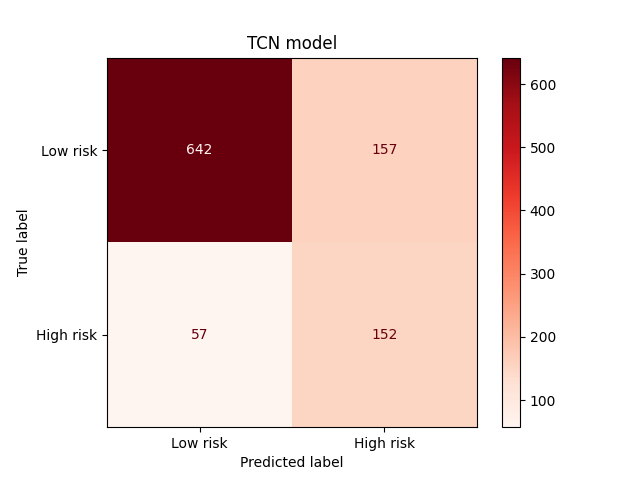}
            \caption[]%
            {{\small On validation sample}}    
            \label{fig9_c}
        \end{subfigure}
        \hfill
        \begin{subfigure}[b]{0.475\textwidth}   
            \centering 
            \includegraphics[width=1.25\textwidth]{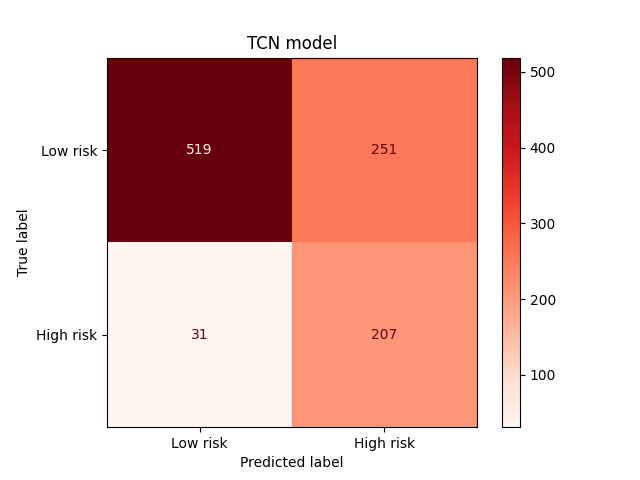}
            \caption[]%
            {{\small On testing sample}}    
            \label{fig9_d}
        \end{subfigure}
        \caption
        {\small Confusion matrices for RL and TCN models} 
        \label{fig9}
    \end{figure}
\newpage

These findings are further supported by the aggregated violation counts reported in Table \ref{tab:var-violation}, where the RL-adjusted VaR consistently reports fewer exceedances than both GARCH-type and TCN models, particularly at the 5\% risk level. The nonparametric Wilcoxon rank-sum test (Table \ref{tab10}) confirms the statistical significance of this difference, with a p-value of 0.0039 indicating robust outperformance in tail coverage. This improved statistical validity is not merely a technical refinement; it has direct consequences for portfolio protection during high-volatility episodes.\\

The source of this enhanced control lies in the accurate classification of the underlying risk states. Figures \ref{fig9} and \ref{fig10} reveal that the RL model achieves high precision in distinguishing between low-risk and high-risk regimes. In the validation and test samples, it correctly identifies 675 and 597 low-risk periods, respectively, allowing substantial downward adjustments to the VaR during tranquil market phases. In parallel, the TCN model proves more effective in detecting high-risk episodes, with 167 and 203 accurate classifications across the two samples, prompting timely increases in capital buffers. The alignment between these classifications and real-world market events, such as the COVID-19 pandemic or post-pandemic geopolitical shocks, highlights the capacity of these models to dynamically recalibrate risk estimates in response to structural breaks.\\

Beyond predictive accuracy, the economic implications of these adjustments are nontrivial. Efficient risk management requires that VaR models not only safeguard against losses but also avoid excessive capital conservatism. In this regard, Figures \ref{fig12} and \ref{fig13} illustrate that the RL and TCN models tend to produce less negative VaR estimates than GJR-GARCH, achieving similar levels of coverage with lower implicit capital requirements. This distributional shift is statistically validated by the Mann–Whitney test (Table \ref{tab11}), which shows that the RL model produces significantly higher (i.e. less conservative) VaR medians across the thresholds 1\% and 5\%, under both validation and test conditions. The ability to reduce capital charges without compromising regulatory compliance constitutes a substantial operational advantage, particularly in capital-constrained environments.

\begin{figure}[H]
        \centering
        \begin{subfigure}[b]{0.475\textwidth}
            \centering
            \includegraphics[width=\textwidth]{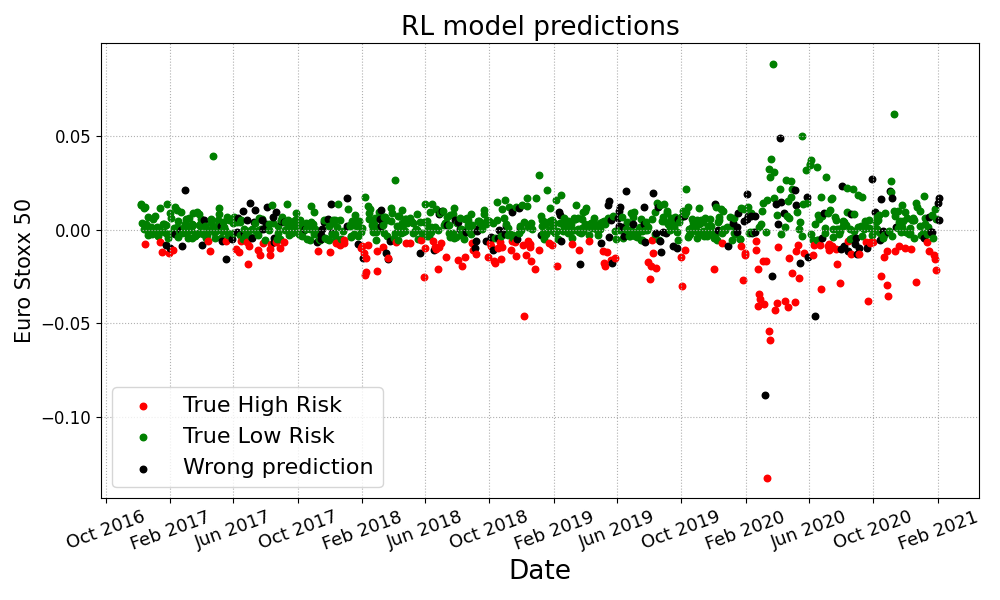}
            \caption[Network2]%
            {{\small On validation sample}}    
            \label{fig10_a}
        \end{subfigure}
        \hfill
        \begin{subfigure}[b]{0.475\textwidth}  
            \centering 
            \includegraphics[width=\textwidth]{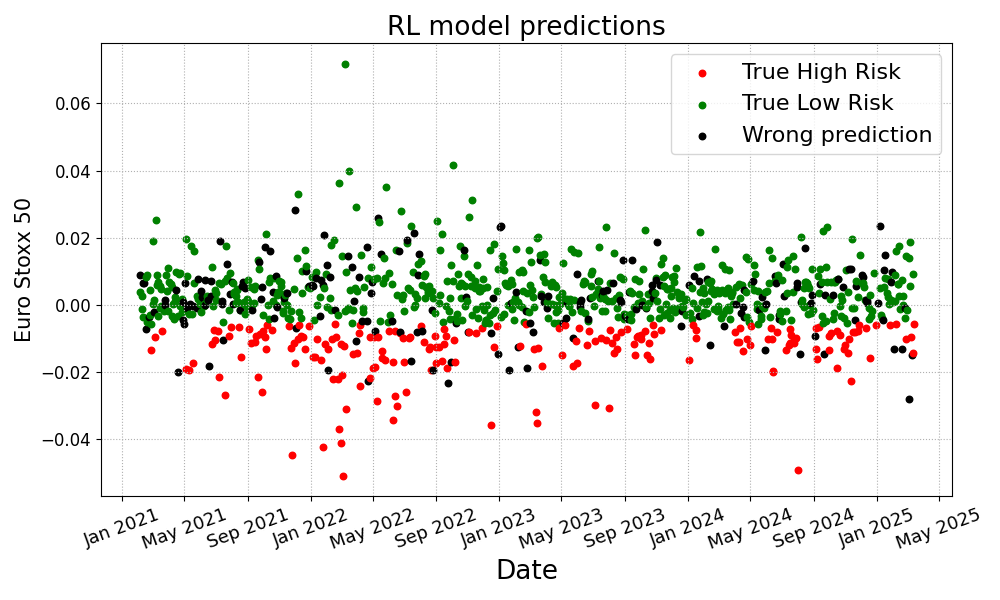}
            \caption[]%
            {{\small On testing sample}}    
            \label{fig10_b}
        \end{subfigure}
        \vskip\baselineskip
        \begin{subfigure}[b]{0.475\textwidth}   
            \centering 
            \includegraphics[width=\textwidth]{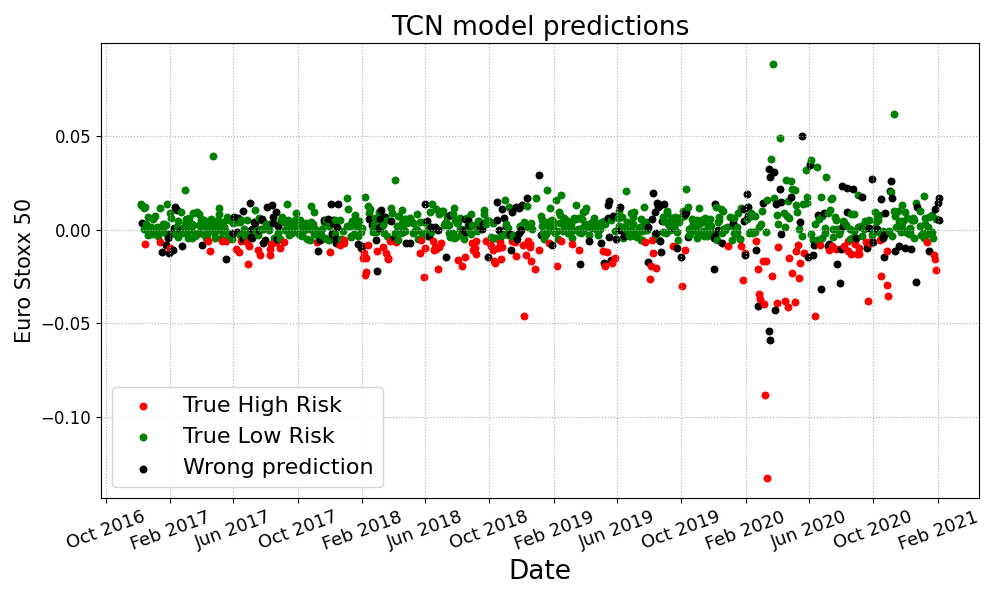}
            \caption[]%
            {{\small On validation sample}}    
            \label{fig10_c}
        \end{subfigure}
        \hfill
        \begin{subfigure}[b]{0.475\textwidth}   
            \centering 
            \includegraphics[width=\textwidth]{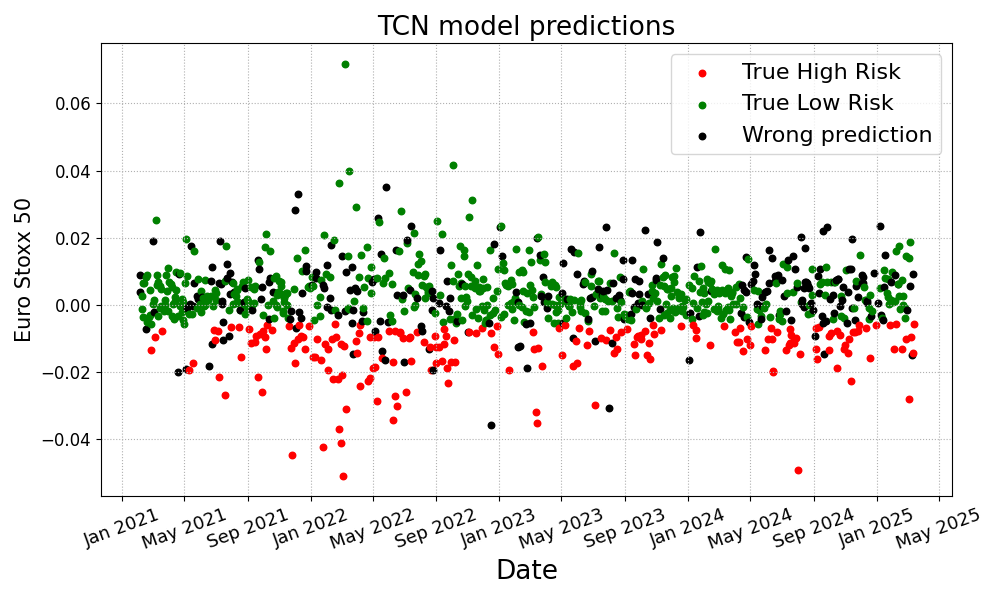}
            \caption[]%
            {{\small On testing sample}}    
            \label{fig10_d}
        \end{subfigure}
        \caption
        {\small Prediction of risk levels} 
        \label{fig10}
    \end{figure}

\newpage
\noindent
\makebox[\textwidth][c]{% Centrage global
    % -------- Figure --------
    \begin{minipage}[b]{0.48\textwidth}
        \centering
        \includegraphics[width=1.1\linewidth]{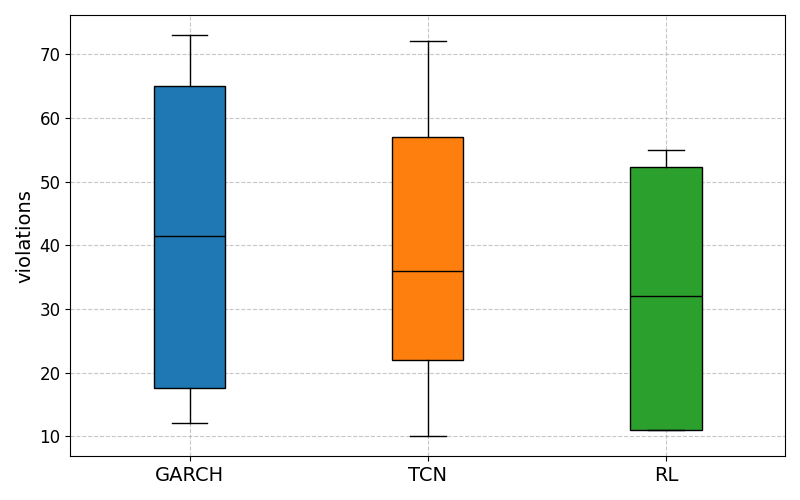} % Ton image ici
        \captionof{figure}{Boxplots of VaR violation}
        \label{fig:var-boxplot}
    \end{minipage}
    \hfill
    % -------- Tableau --------
\begin{minipage}[b]{0.48\textwidth}
 \centering
\resizebox{7.cm}{!}{
\begin{tabular}{|c|l|ccc|}
\hline
 \small{Sample} & \small{Model} & GARCH & TCN & RL \\ 
  \hline
\multirow{4}{*}{\begin{turn}{90} \small{validation}\ \ \end{turn}}  & 
    GARCH 95\% & 73 & 72 & 52 \\ 
  & GJR 95\% & 68 & 72 & 53 \\ 
  & GARCH 99\% & 19 & 25 & 12 \\ 
  & GJR 99\% & 13 & 13 & 11 \\ 
  \hline
\multirow{4}{*}{\begin{turn}{90} \small{testing}\ \end{turn}}  & GARCH 95\% & 64 & 47 & 52 \\ 
  & GJR 95\% & 64 & 52 & 55 \\ 
  & GARCH 99\% & 13 & 13 & 11 \\ 
  & GJR 99\% & 12 & 10 & 11 \\ 
   \hline
\end{tabular}
}
        \captionof{table}{VaR violation}
        \label{tab:var-violation}
    \end{minipage}
}

However, the performance of the TCN model is more nuanced. Although it outperforms traditional benchmarks in the test sample, it exhibits greater sensitivity during periods of extreme volatility, leading to a higher frequency of violations. This relative instability limits its effectiveness under highly turbulent conditions, where the adaptive structure of the RL model is more robust.
Importantly, these performance gains are achieved under challenging empirical conditions.\\ 

The two analyzed subsamples span prolonged episodes of financial instability, including the aftermath of Brexit, the COVID-19 pandemic, rising inflation, and geopolitical turmoil following the Russia–Ukraine conflict. Such environments violate many of the assumptions underlying traditional econometric models, particularly stationarity and linearity, and underscore the need for models capable of learning from non-standard and evolving dynamics.
In this context, the dynamic adjustment of VaR thresholds via classification-informed modulation parameterized by $b_1$ and $b_2$ emerges as a central innovation. It enables the model to scale risk estimates in real time, expanding the buffer in periods of increased uncertainty while relaxing constraints during stable market conditions. This adaptability is especially valuable for institutions that want to align capital reserves with prevailing risk profiles and report risk measures that are both responsive and explainable.
%
%----------------%
%     table 10   %
%----------------%
\begin{table}[H]
\centering
\resizebox{16cm}{!}{
\begin{tabular}{|l|l|c|c|c|}
  \hline
 Test hypotheses & Test indicators & TCN vs GARCH & RL vs GARCH & TCN vs RL \\ 
  \hline
  \multirow{2}{*}{$H_0$ : $\mu_1 = \mu_2$ \mbox{ vs } $H_1$ : $\mu_1 < \mu_2$} & Test statistic   & 12 & 0.0  & 8  \\
 %  $H_1$ : $\mu_1 < \mu_2$
   & P-value   & 0.3997 & 0.0039 & 0.0976  \\
 \hline
\end{tabular}
}
\caption{Equality test for VaR violation means} 
\label{tab10} 
\end{table}
%
%----------------%
%    table 11    %
%----------------%
\begin{table}[H]
\begin{subtable}[c]{0.5\textwidth}
\centering
\resizebox{8.cm}{!}{
\begin{tabular}{l l cc cc}
\toprule
 & & \multicolumn{2}{c}{GARCH} & \multicolumn{2}{c}{GJR-GARCH} \\
\cmidrule(lr){3-4} \cmidrule(lr){5-6}
$\alpha$             & Ind stat  & TCN    & RL     & TCN     & RL \\
\midrule
\multirow{2}{*}{5\%} & statistic & 243568 & 228198 & 241408  & 235833 \\
                     & p-value   & 0.1236 & 0.0136 & 0.08212 & 0.02308 \\
\midrule
\multirow{2}{*}{1\%} & statistic & 252784 & 238663  & 241376  & 236536 \\
                     & p-value   & 0.4362 & 0.04572 & 0.08159 & 0.0436 \\
\bottomrule
\end{tabular}
}
\label{tab11_a}
\subcaption{On validation sample}
\end{subtable}
\begin{subtable}[c]{0.5\textwidth}
\centering
\resizebox{8.cm}{!}{
\begin{tabular}{l l cc cc}
\toprule
 & & \multicolumn{2}{c}{GARCH} & \multicolumn{2}{c}{GJR-GARCH} \\
\cmidrule(lr){3-4} \cmidrule(lr){5-6}
$\alpha$             & Ind stat  & TCN     & RL     & TCN     & RL \\
\midrule
\multirow{2}{*}{5\%} & statistic & 234618  & 226652 & 220530  & 235833 \\
                     & p-value   & 0.01678 & 0.0014 & 0.0001  & 0.02308 \\
\midrule
\multirow{2}{*}{1\%} & statistic & 237250  & 236431  & 238642 & 238250 \\
                     & p-value   & 0.03284 & 0.02685 & 0.0455 & 0.04159 \\
\bottomrule
\end{tabular}
}
\label{tab11_b}
\subcaption{On testing sample}
\end{subtable}
\caption{Mann-Whitney test between distributions of $\text{VaR}_{t+1}(\alpha)$ and $\text{VaR}_{ML}(\alpha)$}
\label{tab11}
\end{table}

\subsection{Regulatory Implications and Methodological Limitations}
The financial crisis of 2007–2009 revealed critical weaknesses in conventional Value-at-Risk (VaR) models, particularly their inability to capture rare but catastrophic events. This motivated a series of reforms by the Basel Committee on Banking Supervision \citep{supervision2011basel, supervision2012basel, basel2017basel}, culminating in the replacement of VaR by a more robust and coherent risk measure: the Expected Shortfall (ES).
\newpage
\noindent Unlike VaR, which only considers a quantile of the return distribution, ES evaluates the conditional expectation of losses beyond the VaR threshold. Formally, the Expected Shortfall at confidence level $\alpha$ is defined as:
$$
ES_{t+1}(\alpha) = \mathbb{E}[r_{t+1} \mid r_{t+1} < \text{VaR}_{t+1}(\alpha)]
$$
where $r_{t+1}$ denotes the future return for the next time period.\\

This makes clear that VaR remains a fundamental component in the regulatory risk architecture, as its estimation directly underpins the calculation of ES. In this context, improving the quality of VaR estimates, especially under non-normal and volatile conditions, is of both theoretical and operational importance.
Our approach contributes to this objective by proposing a hybrid architecture based on Temporal Convolutional Networks (TCN) and deep reinforcement learning (DRL), which offers enhanced flexibility in capturing nonlinear dynamics and regime shifts. These models leverage long-memory structures and adaptive learning mechanisms, making them distinct from traditional GARCH-type frameworks that rely on rigid parametric assumptions.\\

To address regulatory concerns about model transparency, we have incorporated a variable selection step based on the Boruta algorithm, which identifies statistically significant features for classification. This contributes to the interpretability of the system by clarifying the informational basis upon which the risk regime classifier operates. In addition, the risk classification step itself, implemented prior to VaR adjustment, can be interpreted as a structural decision rule separating high- and low-risk periods, supported by confusion matrices and predictive performance metrics.
\begin{figure}[H]
    \centering
    \includegraphics[width=1.\linewidth]{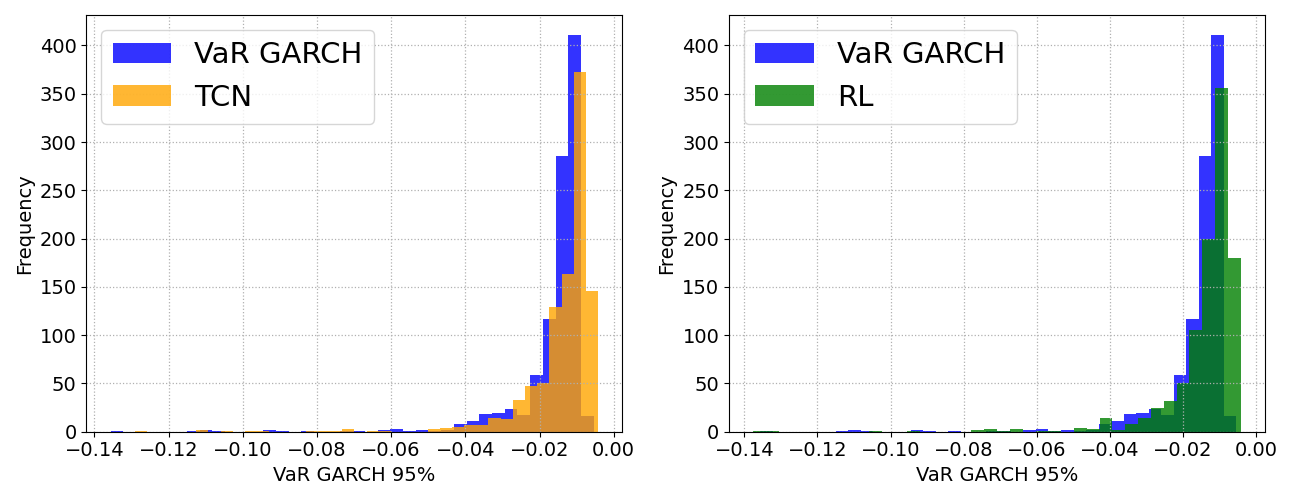}
    \caption{Comparison of VaR distributions with a GJR-GARCH volatility process on the validation sample}
    \label{fig12}
\end{figure}

However, we recognize that further interpretation of our DRL model remains a major challenge. Deep neural network decisions are inherently "black boxes," which is problematic in critical domains such as financial risk management \citep{guidotti2018survey,rudin2019stop}. To address this lack of transparency, we could explore post-hoc explanation tools such as SHAP (SHapley Additive ExPlanations) \citep{lundberg2017unified} or LIME (Local Interpretable Model-Agnostic Explanations)\citep{ribeiro2016should}. Although these techniques apply to all types of model, a particularly promising avenue would be to train a surrogate model, such as a random forest or XGBoost classifier to approximate the behavior of our DRL agent. By analyzing the decision rules of this interpretable surrogate, we could derive an intuitive representation of the underlying logic that governs the risk adjustment process.
Beyond interpretability, the practical implementation of such models poses significant computational challenges. DRL and TCN architectures require careful calibration, high-dimensional optimization, and nontrivial training infrastructure. The neural network architectures and hyperparameters are fully specified in the appendix \ref{sec::implementation}. 

\begin{figure}[H]
    \centering
    \includegraphics[width=1.\linewidth]{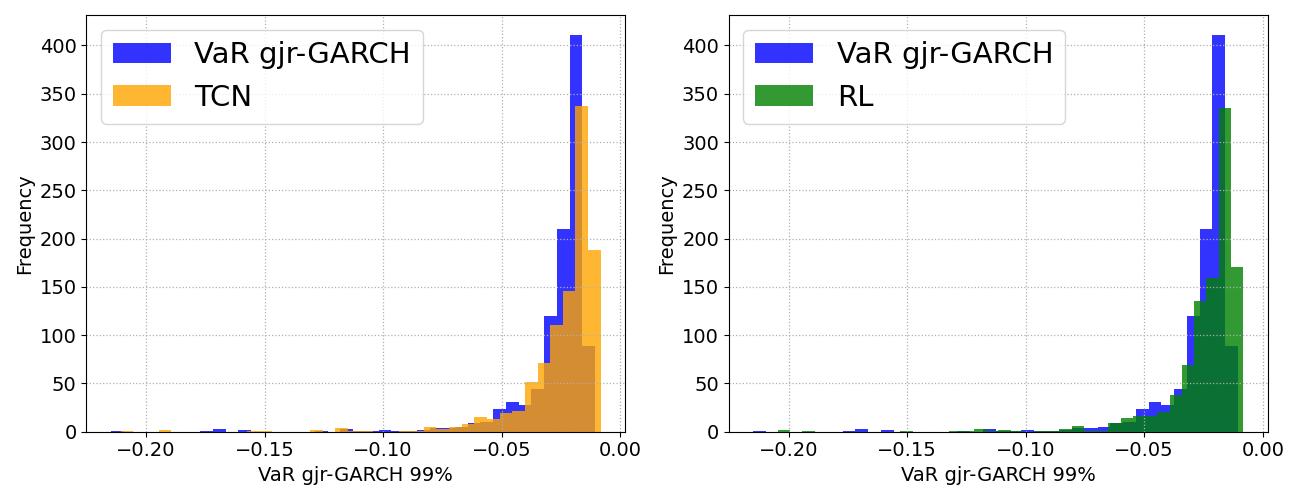}
    \caption{Comparison of VaR distributions with a GJR-GARCH volatility process on the testing sample}
    \label{fig13}
\end{figure}

These demands may hinder their adoption by institutions that lack the technical capacity or regulatory confidence to implement black-box models. However, our empirical results show that these architectures significantly reduce VaR violations while maintaining statistical validity under standard backtesting procedures, including the Kupiec and Christoffersen tests. This performance, combined with their dynamic adaptability, suggests strong potential in active capital management frameworks.
In summation, while deep learning-based approaches open promising perspectives for risk estimation and capital efficiency, their successful adoption in regulatory contexts will depend on continued efforts to enhance their interpretability, transparency, and reproducibility. This includes developing companion tools that explain the logic of risk classifications, quantifying uncertainty in model output, and establishing clear audit trails. These developments are essential to align machine learning innovations with the standards of prudential supervision and the broader expectations of financial governance.

%-------------------------%
%       Conclusion        %
%-------------------------%
\section{Conclusion}
\label{sec5}
This paper introduces a novel framework for Value-at-Risk (VaR) estimation that combines predictive classification with dynamic adjustment mechanisms, enabling more responsive and capital-efficient risk measurement. By reformulating the challenge of market risk forecasting into an imbalanced classification problem addressed through a Double Deep Q-Network (DDQN) reinforcement learning algorithm, our method captures non-nonlinearities, structural changes, and the asymmetric nature of financial return distributions.\\

Empirical analysis of the Euro Stoxx 50 index, covering more than 16 years of data and several episodes of market stress, demonstrates that the proposed classification-adjusted VaR model significantly outperforms traditional volatility-based approaches. Reduce both the number and temporal clustering of violations, passes standard backtesting procedures (Kupiec and Christoffersen), and aligns with the tail behavior predicted by Extreme Value Theory. From an economic standpoint, the model enables more efficient capital allocation by providing less conservative VaR estimates without compromising regulatory coverage.

For practitioners and policymakers, the following strategic takeaways emerge:
\begin{itemize}
\item Risk-level classification enhances tail risk management: Adjusting VaR thresholds according to predicted low- and high-risk periods enables closer alignment of capital buffers with prevailing market conditions.

\item Reinforcement learning enables adaptive forecasting: Unlike static models, the agent continuously updates its classification strategy in response to changing patterns in return dynamics.

\item Compliance with regulatory standards is preserved: The methodology remains compatible with the expected shortfall framework and satisfies key statistical requirements, ensuring relevance in Basel III contexts.

\item Capital efficiency and risk sensitivity can coexist: The approach reduces unnecessary capital charges during stable periods while enhancing protection during volatile episodes.
\end{itemize}
However, several limitations should be acknowledged. Its implementation involves non-trivial computational complexity, sensitivity to hyperparameters, and reduced interpretability compared to classical models. Although this limitation is mitigated through Boruta-based feature selection and structural performance evaluation through classification metrics, more work is needed to improve transparency, especially through the development of surrogate model–based explanation tools to enhance decision auditability.

\section*{Compliance with Ethical Standards}
\subsection*{Declaration of Competing Interest}
All authors declare that they have no conflicts of interest.
\subsection*{Ethical approval:}
This article does not contain any studies with human participants or animals performed by any of the authors.

\bibliography{biblio}
\newpage
%--------------------%
%      Annexes       %
%--------------------%

\setcounter{figure}{0}
\renewcommand{\thefigure}{\Alph{section}\arabic{figure}}

\setcounter{table}{0}
\renewcommand{\thetable}{\Alph{section}\arabic{table}}

\begin{appendix}
\section{Selected features}
%% \label{}

\begin{table}[H]
\centering
\resizebox{14.5cm}{!}{
\begin{tabular}{|l|l|l|l|} 
\hline
n\textsuperscript{\underline{o}}     & Asset         & ID         & Informations \\
\hline
  &   &   &  \\
1  &  AEX                                             & $^{\wedge}$AEX         & Index  \\
2  &  CAC 40                                          & $^{\wedge}$FCHI        & Index  \\
3  &  DAX                                             & $^{\wedge}$GDAXI       & Index  \\
4  &  Euro Stoxx 50                                   & $^{\wedge}$STOXX50E     & Index  \\
  &   &   &  \\
\hline
  &   &   &  \\
5 & Brent Crude Oil Last Day                          & BZ=F         & Commodity \\
6 & EURO/USD                                          & EURUSD=X     & Currency \\
7 & EURO/GBP                                          & EURGBP=X     & Currency \\
  &   &   &  \\
\hline
  &   &   &  \\
8 & SPDR EURO STOXX 50 ETF                            & FEZ          & ETF  \\
  &   &   &  \\
\hline
  &   &   &  \\
9 & Bollinger Bands                                & BB upper \& BB lower   & Technical trading indicator \\
10 & Relative Strength Index                       & RSI 14                 & Technical trading indicator  \\
11 & Standard Moving Average                       & SMA 5 \& SMA 15        & Technical trading indicator  \\
12 & Exponentially Weighted Moving Average         & EMA 5 \& EMA 15        & Technical trading indicator  \\
  &   &   &  \\
\hline
  &   &   &  \\
13 & Forecasting returns using GARCH model         & ${\wedge}$STOXX50E mu GARCH         & Econometric indicators  \\
14 & Forecasting volatility using GARCH model      & ${\wedge}$STOXX50E sig GARCH    & Econometric indicators  \\
15 & Forecasting volatility using GJR6GARCH model  & ${\wedge}$STOXX50E sig gjr-GARCH    & Econometric indicators  \\
16 & Forecasting VaR using GARCH model             & $^{\wedge}$STOXX50E VaR GARCH       & Econometric indicators  \\
17 & Forecasting VaR using GJR-GARCH model         & $^{\wedge}$STOXX50E VaR gjr-GARCH   & Econometric indicators  \\
  &   &   &  \\
\hline
  &   &   &  \\
18 & SMA strategy trading positions                & Signal 1         & Technical trading indicator  \\
19 & EMA strategy trading positions                & Signal 2         & Technical trading indicator  \\
20 & BB strategy trading positions                 & Signal 3         & Technical trading indicator  \\
  &   &   &  \\
\hline
  &   &   &  \\
21 & Forecasting risk level of returns with ARIMA model      & ARIMA indicator    & Class prediction  \\
  &   &   &  \\
\hline
\end{tabular}
}
\caption{Selected features}
\label{tabB1}
\end{table}

\newpage
\section{Confusion matrices}
\begin{figure}[h]
        \centering
        \begin{subfigure}[b]{0.475\textwidth}
            \centering
            \includegraphics[width=1.15\textwidth]{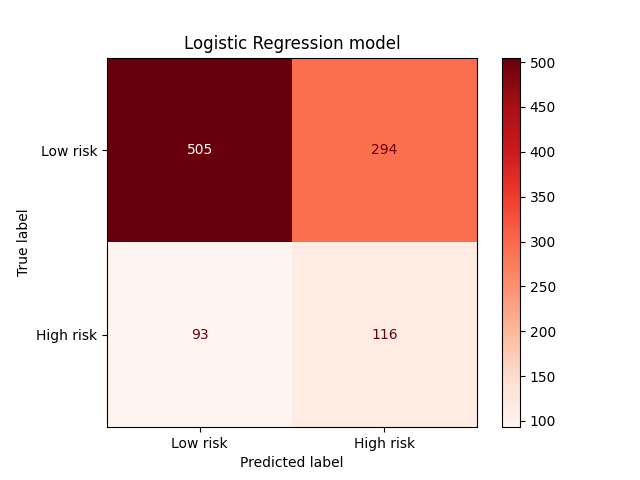}
            \caption[Network2]%
            {{\small On validation sample}}    
            \label{figA1_a}
        \end{subfigure}
        \hfill
        \begin{subfigure}[b]{0.475\textwidth}  
            \centering 
            \includegraphics[width=1.15\textwidth]{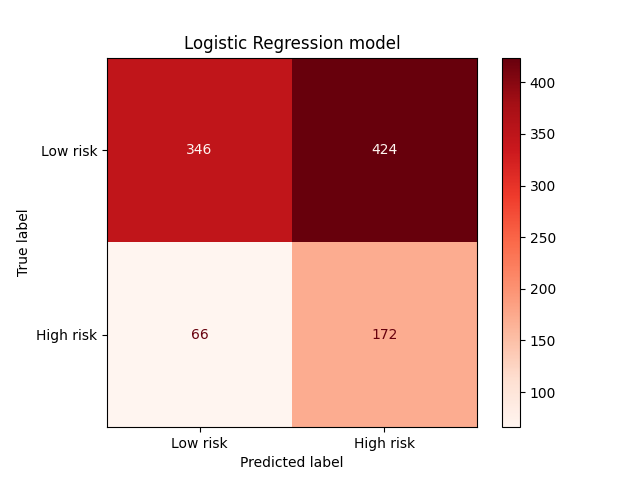}
            \caption[]%
            {{\small On testing sample}}    
            \label{figA1_b}
        \end{subfigure}
        \vskip\baselineskip
        \begin{subfigure}[b]{0.475\textwidth}   
            \centering 
            \includegraphics[width=1.15\textwidth]{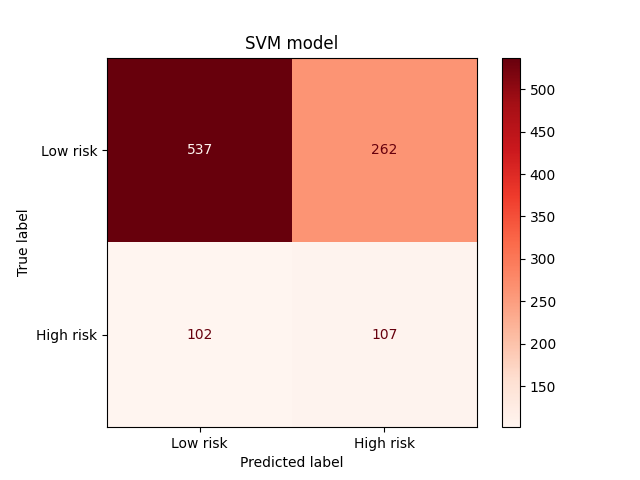}
            \caption[]%
            {{\small On validation sample}}    
            \label{figA1_c}
        \end{subfigure}
        \hfill
        \begin{subfigure}[b]{0.475\textwidth}   
            \centering 
            \includegraphics[width=1.15\textwidth]{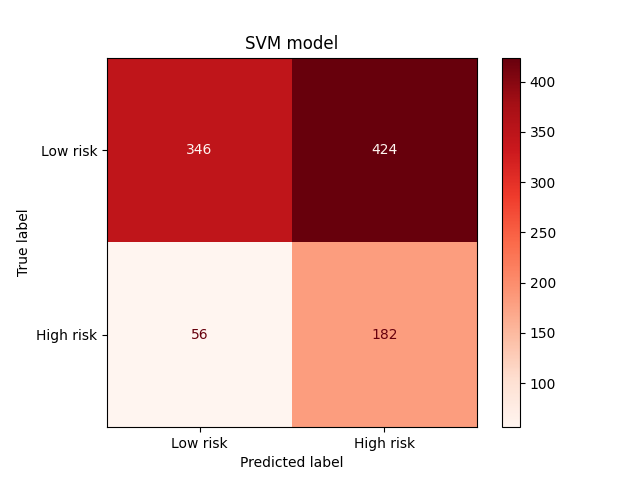}
            \caption[]%
            {{\small On testing sample}}    
            \label{figA1_d}
        \end{subfigure}
        \caption
        {\small Confusion matrices for learning models} 
        \label{figA1}
    \end{figure}
\begin{figure}[H]
        \centering
        \begin{subfigure}[b]{0.475\textwidth}   
            \centering 
            \includegraphics[width=1.15\textwidth]{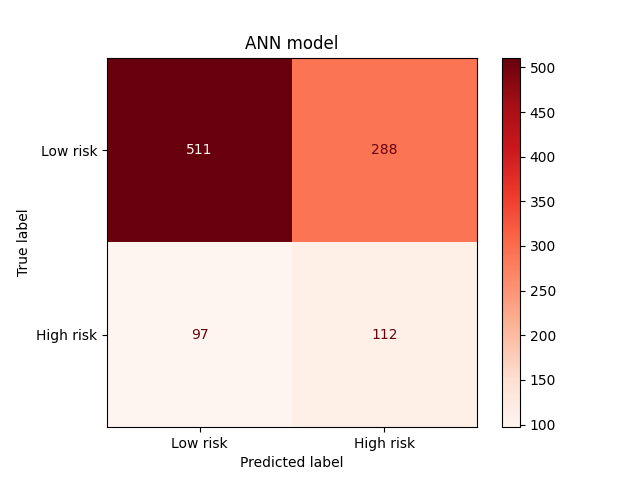}
            \caption[]%
            {{\small On validation sample}}    
            \label{figA1_e}
        \end{subfigure}
        \hfill
        \begin{subfigure}[b]{0.475\textwidth}   
            \centering 
            \includegraphics[width=1.15\textwidth]{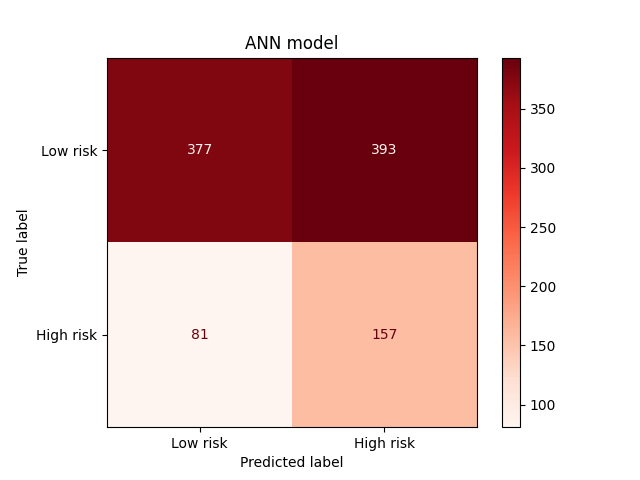}
            \caption[]%
            {{\small On testing sample}}    
            \label{figA1_f}
        \end{subfigure}
        \vskip\baselineskip
        \begin{subfigure}[b]{0.475\textwidth}   
            \centering 
            \includegraphics[width=1.15\textwidth]{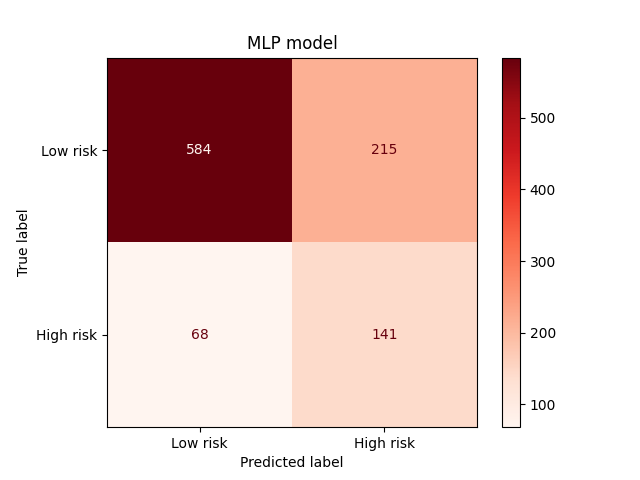}
            \caption[]%
            {{\small On validation sample}}    
            \label{figA1_g}
        \end{subfigure}
        \hfill
        \begin{subfigure}[b]{0.475\textwidth}   
            \centering 
            \includegraphics[width=1.15\textwidth]{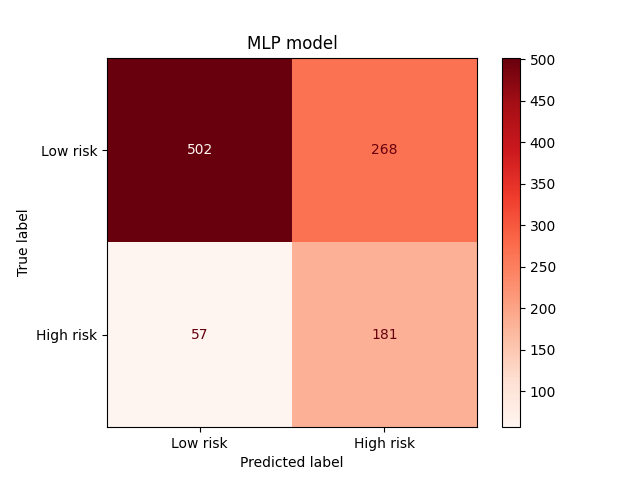}
            \caption[]%
            {{\small On testing sample}}    
            \label{figA1_h}
        \end{subfigure}

        \caption
        {\small Confusion matrices for learning models} 
        \label{figA2}
    \end{figure}
%
%\newpage
%
%
\setcounter{table}{0}
\renewcommand{\thetable}{\Alph{section}\arabic{table}}
\section{VaR robustness tests}

\begin{table}[H]
\begin{subtable}[c]{0.5\textwidth}
\centering
\resizebox{8.5cm}{!}{
\begin{tabular}{|l|l|c|c|c|c|c|}
\hline

\multirow{2}{4em}{Model}               & \multirow{2}{4em}{Version} & \multicolumn{3}{c|}{ GPD parameters} & \multicolumn{2}{c|}{Kolmogorov-Smirnov test}\\
%\hline
     \cline{3-7}             &      &  $\hat{\mu}$ & $\hat{\beta}$ & $\hat{\xi}$& stat & p-value \\
                         \hline
                           & Original &  0.0000    &  0.0067     & 0.2474   & 0.0973  & 0.4647 \\
$VaR^{GARCH}(\alpha=5\%)$  & TCN      &  0.0001    &  0.0035     & 1.4108   & 0.1555  & 0.0669  \\
                           &  RL      &  0.0002    &  0.0025     & 1.4020   & 0.1553  & 0.0641       \\

\hline
                           & Original &  0.0002  &  0.0044   &  1.3963  &  0.1847    & 0.0167 \\
$VaR^{GJR}(\alpha=5\%)$    & TCN      &  0.0001  &  0.0047   &  0.4501  &  0.0677    & 0.8690 \\
                           &  RL      &  0.0001  &  0.0047   & 0.3549   &  0.0878    & 0.6394 \\

\hline
                           & Original &   0.0000    &  0.0059      &  0.5109  & 0.1404  & 0.7988   \\
$VaR^{GARCH}(\alpha=1\%)$  & TCN      &   0.0006    &  0.0025      &  1.3735  & 0.1343  & 0.7082   \\
                           &  RL      &   0.0006    &  0.0032      &  1.3734  & 0.1736  & 0.6234 \\

\hline
                           & Original &   0.0002   &  0.0034  &  1.3992  &  0.1464  & 0.7577  \\
$VaR^{GJR}(\alpha=1\%)$    & TCN      &   0.0009   &  0.0025  &  1.4937  &  0.1554  & 0.5559  \\
                           &  RL      &   0.0009   &  0.0024  &  1.2674  &  0.1314  & 0.9120  \\

\hline

\end{tabular}
}
\label{tabEVT_a}
\subcaption{On validation sample}
\end{subtable}
\begin{subtable}[c]{0.5\textwidth}
\centering
\resizebox{8.5cm}{!}{
\begin{tabular}{|l|l|c|c|c|c|c|}
\hline

\multirow{2}{4em}{Model}               & \multirow{2}{4em}{Version} & \multicolumn{3}{c|}{ GPD parameters} & \multicolumn{2}{c|}{Kolmogorov-Smirnov test}\\
%\hline
     \cline{3-7}             &      &  $\hat{\mu}$ & $\hat{\beta}$ & $\hat{\xi}$& stat & p-value \\
                         \hline
                           & Original &  0.0000  &  0.0049  &  0.4709  &  0.0719  &  0.8714 \\
$VaR^{GARCH}(\alpha=5\%)$  & TCN      &  0.0002  &  0.0040  &  1.4046  &  0.1627  &  0.1484  \\
                           &  RL      &  0.0005  &  0.0043  &  1.2900  &  0.1781  &  0.0691   \\

\hline
                           & Original &   0.0000   &  0.0047  &  0.3958  &  0.0458  &  0.9984   \\
$VaR^{GJR}(\alpha=5\%)$    & TCN      &   0.0001   &  0.0046  &  0.3762  &  0.1020  &  0.6144   \\
                           &  RL      &   0.0006   &  0.0029  &  1.3873  &  0.1503  &  0.1505   \\

\hline
                           & Original &   0.0004   &   0.0033   &   1.2981   &   0.1658   &   0.8120  \\
$VaR^{GARCH}(\alpha=1\%)$  & TCN      &   0.0001   &   0.0060   &   0.4435   &   0.1412   &  0.9258 \\
                           &  RL      &   0.0001   &   0.0056   &   0.4367   &   0.1378   &  0.9668 \\

\hline
                           & Original &   0.0001   &  0.0047  &  1.4028   &  0.2144  &  0.5678 \\
$VaR^{GJR}(\alpha=1\%)$    & TCN      &   0.0008   &  0.0260  &  -1.2738  &  0.2122  &  0.6844 \\
                           &  RL      &   0.0013   &  0.0033  &  1.3297   &  0.2262  &  0.5527 \\

\hline

\end{tabular}
}
\label{tabEVT_a}
\subcaption{On testing sample}
\end{subtable}
\caption{VaR robustness tests}
\label{tabEVT}
\end{table}

\newpage

\setcounter{figure}{0}
\renewcommand{\thefigure}{\Alph{section}\arabic{figure}}

\setcounter{table}{0}
\renewcommand{\thetable}{\Alph{section}\arabic{table}}
\section{Implementation Details}
\label{sec::implementation}

This section summarizes the experimental setup for the models under study Logistic Regression (LR), Support Vector Machine (SVM), Artificial Neural Network (ANN), Multi-Layer Perceptron (MLP), Temporal Convolutional Network (TCN), and the Double Deep Q-Network (DDQN) in a reinforcement learning framework. Computations were performed on a high-performance workstation (AMD Ryzen 9 5900HX, 32 GB RAM, NVIDIA GeForce RTX 3060 GPU) using Python 3.12.8 with standard scientific libraries (NumPy, pandas, scikit-learn, PyTorch, pymc, imbens).\\

The data set, which covers September 1, 2008 to March 13, 2025, is based on the Euro Stoxx 50 index and related market indicators; a complete variable inventory with descriptive statistics appears in Table \ref{tabB1}. The features selected using the Boruta algorithm were split chronologically into training, validation and test sets to avoid anticipation bias. Given substantial scale differences (Table \ref{tab2}), all predictors were rescaled to $[0,1]$ using Min–Max normalization.

\subsection{General Preprocessing and Training Protocol}
The dataset was chronologically split into training, validation, and test subsets to preserve temporal dependence and avoid anticipation bias. Class imbalance was addressed with the ADASYN algorithm, which synthesizes minority-class observations in underrepresented regions of the feature space, thereby improving high-risk state detection. All predictors were scaled to the [0,1] range using Min–Max normalization to enhance numerical stability and accelerate convergence in gradient-based models. Model selection relied on accuracy, F1-score, recall, precision, and G-mean, with early stopping based on validation performance to prevent overfitting.

\subsubsection{About Figure \ref{fig1}}

The parameters for the ARIMA, ARIMA-GARCH, SVM and MLP models were selected through a combination of information criterion minimization (AIC / BIC) and grid search in the validation set. For time series models, ARIMA orders capture the optimal trade-off between model parsimony and predictive fit, while volatility dynamics is incorporated through the GARCH component. For machine learning models, kernel and regularization parameters for SVM and architectural depth and activation functions for MLP were tuned to maximize out-of-sample classification accuracy.
\begin{table}[H]
    \centering
\resizebox{8.5cm}{!}{
\begin{tabular}{|l|c|c|}
\hline
Model   &   Parameters   & Value \\ 
\hline
ARIMA       &  (p, d, q)                    & (1, 0, 2)\\  
ARIMA–GARCH &  (p, d, q)–(P, Q)             & (1, 0, 2)–(1, 1) \\
SVM         &  Kernel / C / $\gamma$        & RBF / 5.82 / 0.023  \\
MLP (Dense) & Layers / Neurons / Activation & 2 / (32, 16) / ReLU \\
\hline
\end{tabular}
}
 \caption{The implementation details of Figure \ref{fig1}}
 \label{tabf1}
\end{table}

\subsubsection{Specification of Model Structures and Training Settings}
This subsection reports the architectural configurations and hyperparameter settings for all models under consideration, including baseline classifiers, neural networks, temporal convolutional architectures, and the reinforcement learning framework.

\begin{table}[H]
\centering
\resizebox{18.cm}{!}{
\begin{tabular}{|l| l| c| l|}
\hline
\textbf{Model} & \textbf{Parameters} & \textbf{ADASYN } & \textbf{Configuration} \\
\midrule
Logistic Regression & Regularization & Yes & L1 penalty, $l_1$\_ratio=0.36, $C = 3.29$ \\
Support Vector Machine (SVM) & Kernel / $C$ / $\gamma$ & Yes & RBF kernel, $C = 10.426$, $\gamma=$'\textit{adaptive}' \\
Single-Layer Perceptron (ANN) & Hidden units / Activation / Optimizer & Yes & 41 neurons, ReLU, Adam (lr = 0.00058) \\
Multi-Layer Perceptron (MLP) & Layers / Neurons / Activation / Dropout & Yes & 2 layers (41, 35), ReLU, dropout = 0.2, Adam (lr = 0.0009) \\
Temporal Convolutional Network (TCN) & Kernel size / Filters / Dropout & Yes & Kernel = 3, 64 filters, dropout = 0.1, Adam (lr = 0.0005) \\
Double Deep Q-Network (DDQN) & Network size / Discount $\gamma$ / $\epsilon$-decay & No & 2 layers (96, 64), $\gamma = 0.95$, $\epsilon$-decay = 0.995, Adam (lr = 0.0005) \\
\bottomrule
\end{tabular}
}
\caption{Summary of Model Architectures, Training Settings, and Hyperparameters}
\label{tab:model_params}
\end{table}
\newpage

Although the complete hyperparameter optimization process implemented through Bayesian optimization required approximately 71 hours, this duration reflects a one-time calibration effort rather than a recurring operational cost. In reinforcement learning frameworks, such search procedures necessarily involve repeated model training for each candidate parameter set, which artificially inflates the reported computation time if it is not distinguished from standard training. Once optimal hyperparameters are identified, the computational burden drops substantially, with average single-run training times ranging from just a few seconds for simpler econometric or machine learning models to several minutes for deep-reinforcement learning architectures. Importantly, inference times across all models remain negligible, well below one-tenth of a second, ensuring that all approaches, including the proposed DDQN, are suitable for real-time deployment in risk monitoring systems. This separation between calibration cost and operational efficiency is essential for a fair evaluation of the practicality of the proposed methodology, particularly in institutional contexts where training is infrequent but inference is continuous.\\

Figure \ref{fig:train} reports these operational training and inference times for all benchmark models alongside the proposed DDQN framework. Training times vary considerably from one model to another, ranging from just 5 seconds for LR to 481 seconds for DDQN, reflecting the increased computational complexity of deep learning-based approaches. In contrast, inference times remain negligible for all models (less than 0.1 seconds), ensuring their practical applicability in real-time risk monitoring. DDQN, although more computationally demanding for training, maintains inference efficiency comparable to simpler models, highlighting its feasibility for operational deployment once trained.

\begin{figure}[H]
    \centering
    \includegraphics[width=0.65\linewidth]{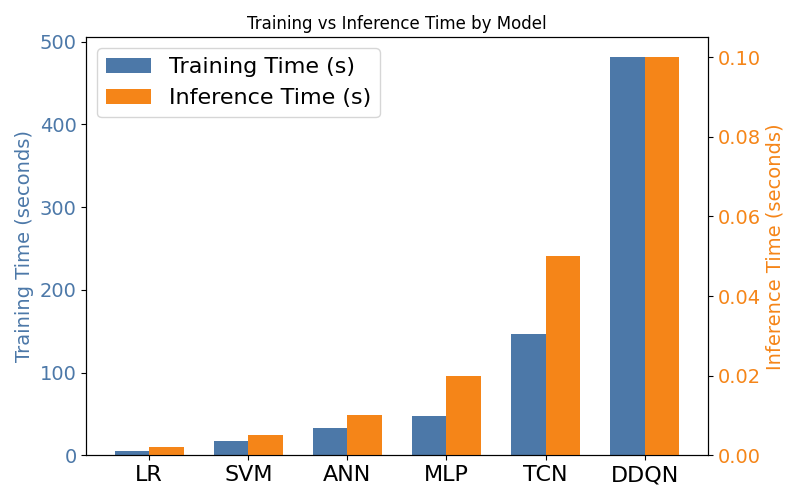}
    \caption{Average operational training times (blue bars, left Y-axis) and inference times (orange bars, right Y-axis) for all benchmark models and the proposed Double Deep Q-Network (DDQN). Reported training times exclude the one-time hyperparameter search phase ($\approx$ 71 hours via Bayesian optimization) to reflect operational deployment conditions. Inference times remain below 0.1 seconds for all models, underscoring their suitability for real-time Value-at-Risk monitoring.}
    \label{fig:train}
\end{figure}

\end{appendix}

\end{document}